%% file: main_apx_dual.tex
\documentclass[11pt]{article}
\usepackage{vitercik}
\usepackage{subfigure}
\usepackage{hyperref}

\newcommand{\score}{\texttt{score}}
\newcommand{\tree}{\pazocal{T}}
\newcommand{\cV}{\pazocal{V}}
\newcommand{\rad}{\mathcal{R}}
\newcommand{\wcrad}{\overline{\rad}}

\begin{document}
	\title{Refined bounds for algorithm configuration:\\
		The knife-edge of dual class approximability}
	\author{
		Maria-Florina Balcan \\ \small Carnegie Mellon University \\ \small \texttt{ninamf@cs.cmu.edu}
		\and 
		Tuomas Sandholm \\ \small Carnegie Mellon University \\
		\small Optimized Markets, Inc.\\
		\small Strategic Machine, Inc.\\
		\small Strategy Robot, Inc.\\
		\small \texttt{sandholm@cs.cmu.edu}
		\and 
		Ellen Vitercik \\ \small Carnegie Mellon University \\ \small \texttt{vitercik@cs.cmu.edu}}
	\date{\today}
	
	\maketitle
	
	\begin{abstract}
		\input{abstract}
	\end{abstract}
\section{Introduction}
\input{intro}

\section{Notation and background}\label{sec:prelim}
\input{notation}

\section{Dual functions}\label{sec:dual}
\input{dual_functions}

\section{Learnability and approximability}\label{sec:learnability}
In this section, we investigate the connection between learnability and approximability. In Section~\ref{sec:UC}, we prove that when the dual functions are approximable under the $L_{\infty}$-norm by simple functions, we can provide strong generalization bounds. In Section~\ref{sec:IP}, we empirically evaluate these improved guarantees in the context of integer programming. Finally, in Section~\ref{sec:not_learnable}, we prove that it is not possible to provide non-trivial generalization guarantees (in the worst case) when the norm under which the dual functions are approximable is the $L_p$-norm for $p < \infty$.

\subsection{Data-dependent generalization guarantees}\label{sec:UC}
\input{UC}

\subsection{Improved integer programming guarantees}\label{sec:IP}
\input{IP}

\subsection{Rademacher complexity lower bound}\label{sec:not_learnable}
\input{not_learnable}

\section{Conclusions}
\input{conclusion}

\bibliography{../../../dairefs}
\bibliographystyle{plainnat}

\onecolumn
\appendix

\section{Notation and learning theory background}\label{app:notation}
\input{learning_theory}

\section{Additional details about learnability and approximability (Section~\ref{sec:learnability})}\label{app:learnablility}
\subsection{Proofs about data-dependent generalization guarantees (Section~\ref{sec:UC})}\label{app:UC}
\input{appendix_generalization}
\input{appendix_SRM}
\subsection{Additional details about improved integer programming guarantees (Section~\ref{sec:IP})}\label{app:IP}
\input{appendix_IP}
\subsection{Proofs about Rademacher complexity lower bound (Section~\ref{sec:not_learnable})}\label{app:not_learnable}
\input{appendix_not_learnable}
\subsection{Connection to statistical learnability}\label{app:statistical}
\input{appendix_statistical}

\end{document}

%% file: abstract.tex
Automating algorithm configuration is growing increasingly necessary as algorithms come with more and more tunable parameters. It is common to tune parameters using machine learning, optimizing performance metrics such as runtime and solution quality. The training set consists of problem instances from the specific domain at hand. We investigate a fundamental question about these techniques: how large should the training set be to ensure that a parameter's average empirical performance over the training set is close to its expected, future performance? We answer this question for algorithm configuration problems that exhibit a widely-applicable structure: the algorithm's performance as a function of its parameters can be approximated by a ``simple'' function.
We show that if this approximation holds under the $L^{\infty}$-norm, we can provide strong sample complexity bounds. On the flip side, if the approximation holds only under the $L^p$-norm for $p < \infty$, it is not possible to provide meaningful sample complexity bounds in the worst case. We empirically evaluate our bounds in the context of integer programming, one of the most powerful tools in computer science. Via experiments, we obtain sample complexity bounds that are up to 700 times smaller than the previously best-known bounds~\citep{Balcan18:Learning}.

%% file: intro.tex
Algorithms typically have tunable parameters that significantly impact their performance, measured in terms of runtime, solution quality, and so on. Machine learning is often used to automate parameter tuning~\citep{Horvitz01:Bayesian,Hutter09:ParamILS,Kadioglu10:ISAC,Sandholm13:Very-Large-Scale}: given a \emph{training set} of problem instances from the application domain at hand, this automated procedure returns a parameter setting that will ideally perform well on future, unseen instances.

It is important to be careful when using this automated approach: if the training set is too small, a parameter setting with strong average empirical performance over the training set may have poor future performance on unseen instances.
\emph{Generalization bounds} provide guidance when it comes to selecting the training set size. They bound the difference between an algorithm's performance on average over the training set (drawn from an unknown, application-specific distribution) and its expected performance on unseen instances.
 These bounds can be used to evaluate a parameter setting returned by any black-box procedure: they bound the difference between that parameter's average performance on the training set and its expected performance.

At a high level, we provide generalization bounds that hold when an algorithm's performance as a function of its parameters exhibits a widely-applicable structure: it can be approximated by a ``simple'' function. We prove that it is possible to provide strong generalization bounds when the approximation holds under the $L^{\infty}$-norm. Meanwhile, it is not possible to provide strong guarantees in the worst-case if the approximation only holds under the $L^p$-norm for $p < \infty$. Therefore, this connection between learnability and approximability is balanced on a knife-edge.

Our analysis is based on structure exhibited by \emph{primal} and \emph{dual} functions~\citep{Assouad83:Densite}, which we now describe at a high level. To provide generalization bounds, a common strategy is to bound the \emph{intrinsic complexity} of the following function class $\cF$: for every parameter vector $\vec{r}$ (such as a CPLEX parameter setting) there is a function $f_{\vec{r}} \in \cF$ that takes as input a problem instance $x$ (such as an integer program) and returns $f_{\vec{r}}(x)$, the algorithm's \emph{performance} on input $x$ when parameterized by $\vec{r}$. Performance is measured by runtime, solution quality, or some other metric. The functions $f_{\vec{r}}$ are called \emph{primal functions}.

The class $\cF$ is gnarly: in the case of integer programming algorithm configuration, the domain of every function in $\cF$ consists of integer programs, so it is unclear how to visualize or plot these functions, and there are no obvious notions of Lipschitzness or smoothness to rely on. Rather than fixing a parameter setting $\vec{r}$ and varying the input $x$ (as under the function $f_{\vec{r}}$), it can be enlightening to instead fix the input $x$ and analyze the algorithm's performance as a function of $\vec{r}$. This \emph{dual function} is denoted $f_x^*(\vec{r})$.
The dual functions have a simple, Euclidean domain, they are typically easy to plot, and they often have ample structure we can use to bound the intrinsic complexity of the class $\cF$.

\begin{figure}
	\centering
	\subfigure[\label{fig:apx_pwc_pwc}]{\includegraphics{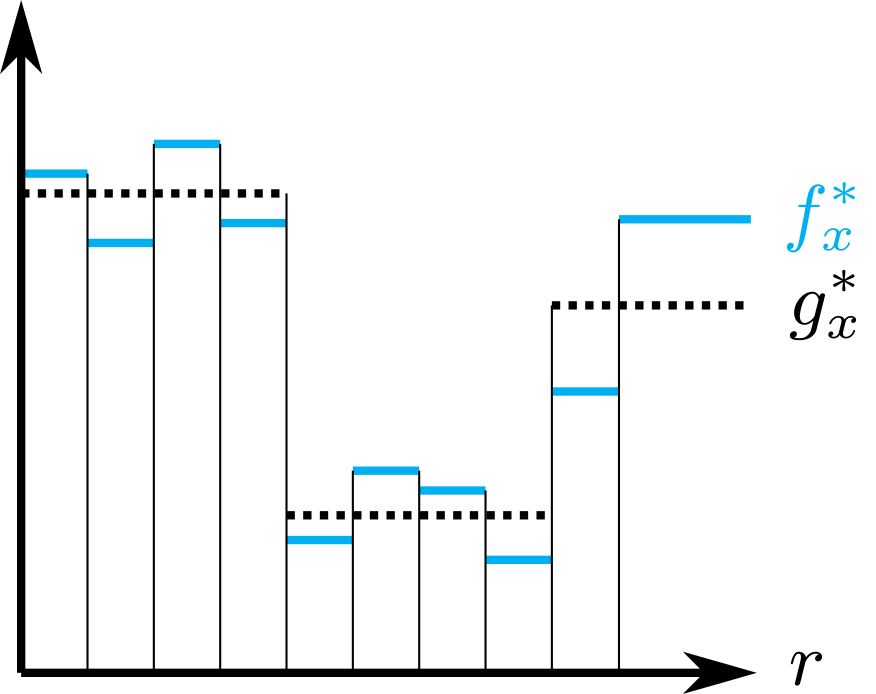}}\qquad
	\subfigure[]{\includegraphics{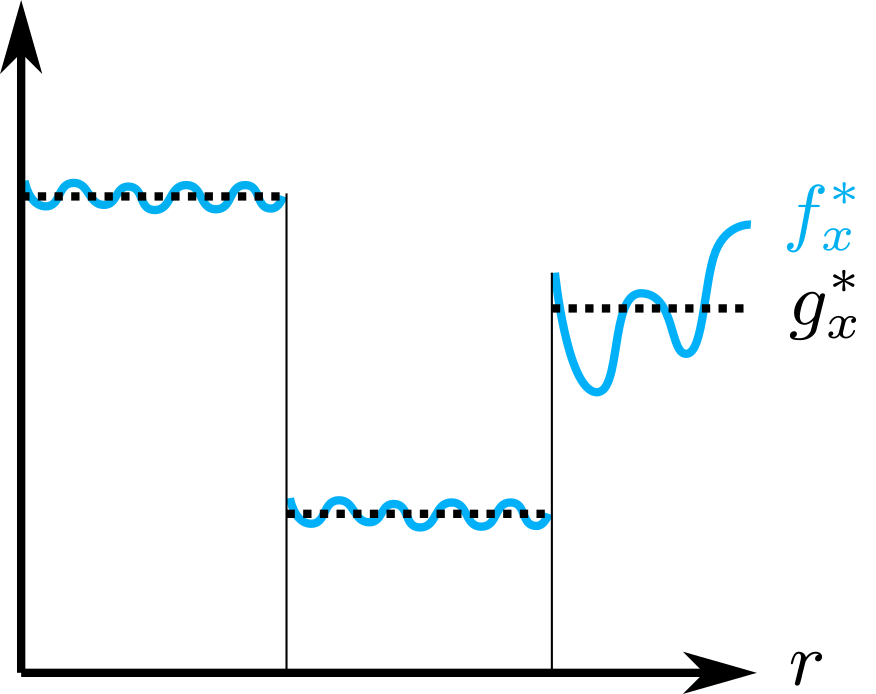}}\qquad
	\subfigure[]{\includegraphics{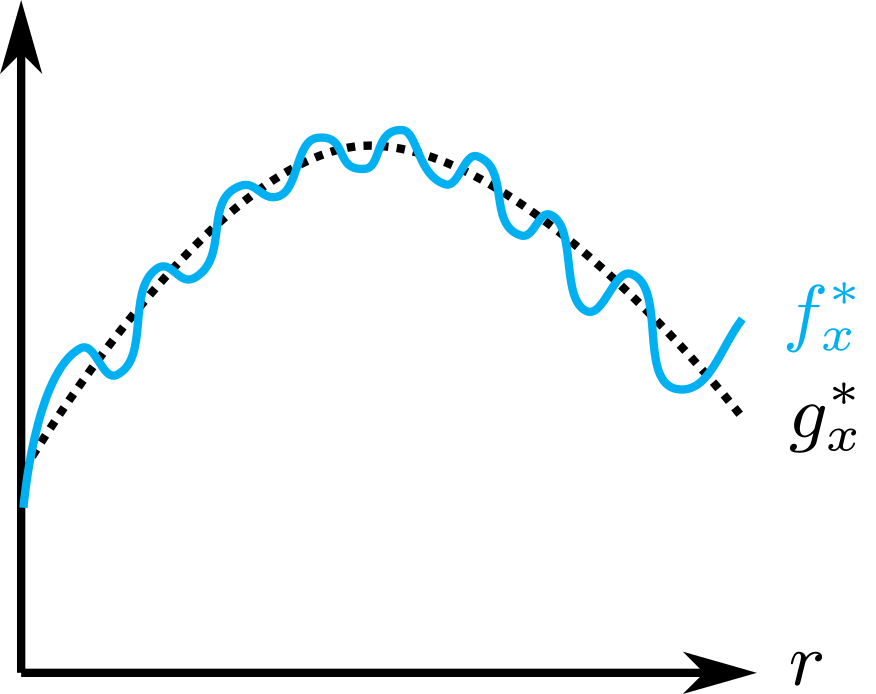}}
	\caption{Examples of dual functions $f_x^* : \R \to \R$ (solid blue lines) which are approximated by simpler functions $g_x^*$ (dotted black lines).}
	\label{fig:apx}
\end{figure}
\paragraph{Our contributions.} We observe that for many configuration problems, the dual functions can be closely approximated by ``simple'' functions, as in Figure~\ref{fig:apx}. This raises the question: can we exploit this structure to provide strong generalization guarantees?
We show that if the dual functions are approximated by simple functions under the $L^{\infty}$-norm (meaning the maximum distance between the functions is small), then we can provide strong generalization guarantees.
However, this is no longer true when the approximation only holds under the $L^p$-norm for $p < \infty$: we present a set of functions whose duals are well-approximated by the simple constant function $g(x) = \frac{1}{2}$ under the $L^p$-norm (meaning $\sqrt[p]{\int \left|f_x^*(\vec{r}) - \frac{1}{2}\right|^p \, d\vec{r}}$ is small), but which are not learnable.

We provide an algorithm that finds approximating simple functions in the following widely-applicable setting: the dual functions are piecewise-constant with a large number of pieces, but can be approximated by simpler piecewise-constant functions with few pieces, as in Figure~\ref{fig:apx_pwc_pwc}. This is the case in our integer programming experiments.

In our experiments, we demonstrate significant practical implications of our analysis. We configure CPLEX, one of the most widely-used integer programming solvers. Integer programming has diverse applications throughout science. Prior research has shown that the dual functions associated with various CPLEX parameters are piecewise constant and has provided generalization bounds that grow with the number of pieces~\citep{Balcan18:Learning}. However, the number of pieces can be so large that these bounds can be quite loose. We show that these dual functions can be approximated under the $L^{\infty}$-norm by simple functions (as in Figure~\ref{fig:apx_pwc_pwc}), so our theoretical results imply strong generalization guarantees.
In our experiments, we demonstrate that in order to obtain the same generalization bound, the training set size required under our analysis is up to 700 times smaller than that of \citet{Balcan18:Learning}. Improved sample complexity guarantees imply faster learning algorithms, since the learning algorithm needs to analyze fewer training instances.

\paragraph{Related research.}
In algorithm configuration,
several papers have provided generalization guarantees for specific algorithm families, including greedy algorithms~\citep{Gupta17:PAC, Balcan18:Dispersion}, clustering algorithms~\citep{Balcan17:Learning, Balcan18:Data, Balcan20:LearningToLink}, and integer programming algorithms~\citep{Balcan18:Learning}. In contrast, we provide general guarantees that apply to any configuration problem that satisfies a widely-applicable structure: the dual functions are approximately simple. A strength of our results is that they are not tied to any specific algorithm family, though we show that our guarantees can be empirically much stronger than the best-known bounds. \citet{Balcan19:How} show that if the dual functions are simple---for example, they are piecewise-constant with few pieces---then it is possible to provide strong generalization bounds. We observe, however, that often the dual functions themselves are not particularly simple, but can be approximated by simple functions. We exploit this  structure to provide more general guarantees. The analysis tools from prior research do not apply to this more general structure, so we require new, refined proof techniques.

Our guarantees are configuration-procedure-agnostic: no matter how one tunes the parameters using the training set, we bound the difference between the resulting parameter setting's performance on average over the training set and its expected performance on unseen instances.
A related line of research has provided learning-based algorithm configuration procedures with provable guarantees~\citep{Kleinberg17:Efficiency, Kleinberg19:Procrastinating,Weisz18:LEAPSANDBOUNDS, Weisz19:CapsAndRuns, Balcan20:Learning}. Unlike the results in this paper, their guarantees are not configuration-procedure-agnostic: they apply to the specific configuration procedures they propose. Moreover, their procedures only apply to finding configurations that minimize computational resource usage, such as runtime, whereas the guarantees in this paper apply to more general measures of algorithmic performance, such as solution quality.

A related line of research has studied integer programming algorithm configuration~\citep{Hutter09:ParamILS, Sabharwal12:Guiding, Sandholm13:Very-Large-Scale,He14:Learning, Balafrej15:Multi, Khalil16:Learning,Khalil17:Learning, DiLiberto16:Dash, Lodi17:Learning, Alvarez17:Machine, Kruber17:Learning, Balcan18:Learning}, as do we, though our results apply more generally. The results in these papers are primarily empirical, with the exception of the paper by \citet{Balcan18:Learning}, with which we compare extensively in Section~\ref{sec:IP}.

%% file: notation.tex
 We study functions that map an abstract domain $\cX$ to $[0,1]$. We  denote the set of all such functions as $[0,1]^{\cX}$. 
 The learning algorithms we analyze have sample access to an unknown distribution $\dist$ over examples $x \in \cX$ and aim to find a function $f \in \cF$ with small expected value $\E_{x \sim \dist}[f(x)]$.

\subsection{Problem definition}
We provide \emph{generalization guarantees}, which bound the difference between the expected value $\E_{x \sim \dist}[f(x)]$ of any function $f \in \cF$ and its empirical average value $\frac{1}{N}\sum_{i = 1}^N f\left(x_i\right)$ over a training set $x_1, \dots, x_N \sim \dist$. 
We focus on functions that are parameterized by a set of vectors $\cR \subseteq \R^d$. Given a vector $\vec{r} \in \cR$, we denote the corresponding function as $f_{\vec{r}} : \cX \to [0,1]$, and we define $\cF = \left\{f_{\vec{r}} \mid \vec{r} \in \cR\right\}$.

Generalization guarantees are particularly useful for analyzing the expected performance of empirical risk minimization learning algorithms for the following reason. Suppose we know that for any function $f \in \cF$, $\left|\frac{1}{N}\sum_{i = 1}^N f\left(x_i\right) - \E_{x \sim \dist}[f(x)]\right| \leq \epsilon$. Let $\hat{f}$ be the function in $\cF$ with smallest average value over the training set: $\hat{f} = \argmin_{f \in \cF}\sum_{i = 1}^N f\left(x_i\right)$. Then $\hat{f}$ has nearly optimal expected value: $\E_{x \sim \dist}\left[\hat{f}(x)\right] - \min_{f \in \cF} \E_{x \sim \dist}[f(x)] \leq 2\epsilon.$
 
 \subsection{Integer programming algorithm configuration}\label{sec:IP_intro}
 
We use integer programming algorithm configuration as a running example, though our results are much more general. An \emph{integer program (IP)} is defined by a matrix $A \in \R^{m \times n}$, a constraint vector $\vec{b} \in \R^m$, an objective vector $\vec{c}\in \R^n$, and a set of indices $I \subseteq [n]$. The goal is to find a vector $\vec{z} \in \R^n$ such that $\vec{c}\cdot \vec{z}$ is maximized, subject to the constraints that $A\vec{z} \leq \vec{b}$ and for every index $i \in I$,  $z_i \in \{0,1\}$.

In our experiments, we tune the parameters of branch-and-bound (B\&B)~\citep{Land60:Automatic}, the most widely-used algorithm for solving IPs. It is used under the hood by commercial solvers such as CPLEX and Gurobi. 
We provide a brief, high-level overview of B\&B, and refer the reader to the textbook by \citet{Nemhauser99:Integer} for more details. B\&B builds a search tree to solve an input IP $x$. At the tree's root is the original IP $x$. At each round, B\&B chooses a leaf of the search tree, which represents an IP $x'$. It does so using a \emph{node selection policy}; common choices include depth- and best-first search. Then, it chooses an index $i \in I$ using a \emph{variable selection policy}. It next \emph{branches} on $z_i$: it sets the left child of $x'$ to be that same integer program $x'$, but with the additional constraint that $z_i = 0$, and it sets the right child of $x'$ to be that same integer program, but with the additional constraint that $z_i = 1$. It solves both LP relaxations, and if either solution satisfies the integrality constraints of the original IP $x$, it constitutes a feasible solution to $x$. B\&B \emph{fathoms} a leaf---which means that it never will branch on that leaf---if it can guarantee that the optimal solution does not lie along that path. B\&B terminates when it has fathomed every leaf. At that point, we can guarantee that the best solution to $x$ found so far is optimal.
In our experiments, we tune the parameters of the variable selection policy, which we describe in more detail in Section~\ref{sec:IP}.

In this setting, $\cX$ is a set of IPs and the functions in $\cF$ are parameterized by CPLEX parameter vectors $\vec{r} \in \R^d$, denoted $\cF = \left\{f_{\vec{r}} \mid \vec{r} \in \R^d\right\}.$ In keeping with prior work~\citep{Balcan18:Learning}, $f_{\vec{r}}(x)$ equals the size of the B\&B tree CPLEX builds given the parameter setting $\vec{r}$ and input IP $x$, normalized to fall in $[0,1]$. The learning algorithms we study take as input a training set of IPs sampled from $\dist$ and return a parameter vector. Since our goal is to minimize tree size, ideally, the size of the trees CPLEX builds using that parameter setting should be small in expectation over $\dist$.

%% file: dual_functions.tex
Our goal is to provide generalization guarantees for the function class $\cF= \left\{f_{\vec{r}} \mid \vec{r} \in \cR\right\}$.
To do so, we use structure exhibited by the \emph{dual function class}. Every function in the dual class is defined by an element $x \in \cX$, denoted $f_x^* : \cR \to [0,1]$. Naturally, $f_x^*(\vec{r}) = f_{\vec{r}}(x)$. The dual class $\cF^* = \left\{f_x^* \mid x \in \cX\right\}$ is the set of all dual functions.

The dual functions are intuitive in our integer programming example. For any IP $x$, the dual function $f_x^*$ measures the size of the tree CPLEX builds (normalized to lie in the interval $[0,1]$) when given $x$ as input, as a function of the CPLEX parameters.
Duals are also straightforward in more abstract settings: if $\cX = \R^d$ and $\cF$ is the set of linear functions $f_{\vec{r}}(\vec{x}) = \vec{r} \cdot \vec{x}$, each dual function $f_{\vec{x}}^*(\vec{r}) = \vec{r} \cdot \vec{x}$ is also linear. When $\cF$ consists of the constant functions $f_{\vec{r}}(x) = \vec{r}$, each dual function is the identity function $f_x^*(\vec{r}) = \vec{r}$.

Prior research shows that when the dual functions are simple---for example, they are piecewise-constant with a small number of pieces---it is possible to provide strong generalization bounds~\citep{Balcan19:How}. 
In many settings, however, we find that the dual functions themselves are not simple, but are approximated by simple functions, as in Figure~\ref{fig:apx}. We formally define this concept as follows.
\begin{definition}[$(\gamma, p)$-approximate]
	Let $\cF = \left\{f_{\vec{r}} \mid \vec{r} \in \cR\right\}$ and $\cG = \left\{g_{\vec{r}} \mid \vec{r} \in \cR\right\}$ be two sets of functions mapping $\cX$ to $[0,1]$. We assume that all dual functions $f_x^*$ and $g_x^*$ are integrable over the domain $\cR$. We say that the dual class $\cG^*$ \emph{$(\gamma, p)$-approximates} the dual class $\cF^*$ if for every element $x$, the distance between the functions $f_x^*$ and $g_x^*$ is at most $\gamma$ under the $L^p$-norm. For $p \in [1, \infty)$, this means that $\norm{f_x^* - g_x^*}_p := \sqrt[p]{\int_{\cR} \left|f_x^*(\vec{r}) - g_x^*(\vec{r})\right|^p \, d \vec{r}} \leq \gamma$ and when $p = \infty$, this means that $\norm{f_x^* - g_x^*}_{\infty} :=\sup_{\vec{r} \in \cR} \left|f_x^*(\vec{r}) - g_x^*\left(\vec{r}\right)\right| \leq \gamma.$
\end{definition}

%% file: UC.tex
We now show that if the dual class $\cF^*$ is $(\gamma, \infty)$-approximated by the dual of a ``simple'' function class $\cG$, we can provide strong generalization bounds for the class $\cF$.
There are many different tools for measuring how ``simple'' a function class is. We use \emph{Rademacher complexity}~\citep{Koltchinskii01:Rademacher}, which intuitively measures the extent to which functions in $\cF$ match random noise vectors $\vec{\sigma} \in \{-1,1\}^N$. 
\begin{definition}[Rademacher complexity]
The \emph{empirical Rademacher complexity} of a function class $\cF = \left\{f_{\vec{r}} \mid \vec{r} \in \cR\right\}$ given a set $\sample = \left\{x_1, \dots, x_N\right\} \subseteq \cX$ is \[\erad(\cF) = \frac{1}{N} \E_{\vec{\sigma} \sim \{-1,1\}^N} \left[\sup_{\vec{r} \in \cR} \sum_{i = 1}^N \sigma_i f_{\vec{r}}\left(x_i\right)\right],\] where each  $\sigma_i$ equals $-1$ or 1 with equal probability.
\end{definition}
The summation $\sum_{i = 1}^N \sigma_i f_{\vec{r}}\left(x_i\right)$  measures the correlation between the random noise vector $\vec{\sigma}$ and the vector $\left(f_{\vec{r}}\left(x_1\right), \dots, f_{\vec{r}}\left(x_N\right)\right)$. By taking the supremum over all parameter vectors $\vec{r} \in \cR$, we measure how well functions in the class $\cF$ correlate with $\vec{\sigma}$ over the sample $\sample$. Therefore, $\erad(\cF)$ measures how well functions in the class $\cF$ correlate with random noise on average over $\sample$. Rademacher complexity thus provides a way to measure the intrinsic complexity of $\cF$ because the more complex the class $\cF$ is, the better its functions can correlate with random noise.
For example, if the class $\cF$ consists of just a single function, $\erad(\cF) = 0$. At the other extreme, if $\cX = [0,1]$ and $\cF = [0,1]^{[0,1]}$ consists of all functions mapping $[0,1]$ to $[0,1]$, $\erad(\cF) = \frac{1}{2}$.

Classic learning-theoretic results provide guarantees based on Rademacher complexity, such as the following.
\begin{theorem}[e.g., \citet{Mohri12:Foundations}]\label{thm:rad_gen}
	For any $\delta \in (0,1)$, with probability $1-\delta$ over the draw of $N$ samples $\sample = \left\{x_1, \dots, x_N\right\} \sim \dist^N$, for all functions $f_{\vec{r}} \in \cF$, 
\[\left|\frac{1}{N} \sum_{i = 1}^N f_{\vec{r}}\left(x_i\right) - \E\left[f_{\vec{r}}(x)\right]\right|
= O \left(\erad(\cF) + \sqrt{\frac{1}{N}\log \frac{1}{\delta}}\right).\]
\end{theorem}
Theorem~\ref{thm:rad_gen} is a \emph{generalization guarantee} because it measures the extent to which a function's empirical average value over the samples generalizes to its expected value.

Ideally, $\erad(\cF)$ converges to zero as the sample size $N$ grows so the bound in Theorem~\ref{thm:rad_gen} also converges to zero. If the class $\cF$ consists of just a single function, $\erad(\cF) = 0$, and Theorem~\ref{thm:rad_gen} recovers Hoeffding's bound. If, for example, $\cX = [0,1]$ and $\cF = [0,1]^{[0,1]}$, $\erad(\cF) = \frac{1}{2}$, and the bound from Theorem~\ref{thm:rad_gen} is meaningless.

We show that if the dual class $\cF^*$ is $(\gamma, \infty)$-approximated by the dual of a class $\cG$ with small Rademacher complexity, then the Rademacher complexity of $\cF$ is also small.
 The full proof of the following theorem in Appendix~\ref{app:UC}.

\begin{restatable}{theorem}{mainRad}\label{thm:rad}
	Let $\cF = \left\{f_{\vec{r}} \mid \vec{r} \in \cR\right\}$ and $\cG = \left\{g_{\vec{r}} \mid \vec{r} \in \cR\right\}$ consist of functions mapping $\cX$ to $[0,1]$. For any $\sample \subseteq \cX$, $\erad(\cF) \leq \erad(\cG) + \frac{1}{|\sample|} \sum_{x \in \sample}  \norm{f_{x}^* - g_{x}^*}_{\infty}.$
\end{restatable}
\begin{proof}[Proof sketch]
	To prove this theorem, we use the fact that for any parameter vector $\vec{r} \in \cR$, any element $x \in \cX$, and any binary value $\sigma \in \{-1,1\}$, $\sigma f_{\vec{r}}(x) = \sigma f_x^*(\vec{r}) \leq \sigma g_x^*(\vec{r}) + \norm{f_x^* - g_x^*}_{\infty} = \sigma g_{\vec{r}}(x) + \norm{f_x^* - g_x^*}_{\infty}.$
\end{proof}

If the class $\cG^*$ $(\gamma, \infty)$-approximates the class $\cF^*$, then $\frac{1}{|\sample|} \sum_{x \in \sample}  \norm{f_{x}^* - g_{x}^*}_{\infty}$ is at most $\gamma$. If this term is smaller than $\gamma$ for most  sets $\sample \sim \dist^N$, then the bound on $\erad(\cF)$ in Theorem~\ref{thm:rad} will often be even better than $\erad(\cG) + \gamma$.

Theorems~\ref{thm:rad_gen} and \ref{thm:rad} imply that with probability $1-\delta$ over the draw of the set $\sample \sim \dist^N$, for all parameter vectors $\vec{r} \in \cR$, the difference between the empirical average value of $f_{\vec{r}}$ over $\sample$ and its expected value is at most $\tilde O\left(\frac{1}{N} \sum_{x \in \sample}  \norm{f_{x}^* - g_{x}^*}_{\infty} + \erad(\cG) + \sqrt{\frac{1}{N}}\right)$.
In our integer programming experiments, we show that this data-dependent generalization guarantee can be much tighter than the best-known worst-case guarantee.

\paragraph{Algorithm for finding approximating functions.} We provide a dynamic programming (DP) algorithm (Algorithm~\ref{alg:DP} in Appendix~\ref{app:IP}) for the widely-applicable case where the dual functions $f_x^*$ are piecewise constant with a large number of pieces. Given an integer $k$, the algorithm returns a piecewise-constant function $g_x^*$ with at most $k$ pieces such that $\norm{f_x^* - g_x^*}_{\infty}$ is minimized, as in Figure~\ref{fig:apx_pwc_pwc}. Letting $t$ be the number of pieces in the piecewise decomposition of $f_x^*$, the DP algorithm runs in $O\left(kt^2\right)$ time. As we describe in Section~\ref{sec:IP}, when $k$ and $\norm{f_x^* - g_x^*}_{\infty}$ are small, Theorem~\ref{thm:rad} implies strong guarantees. We use this DP algorithm in our integer programming experiments.

\paragraph{Structural risk minimization.}
Theorem~\ref{thm:rad} illustrates a fundamental tradeoff in machine learning. The simpler the class $\cG$, the smaller its Rademacher complexity, but---broadly speaking---the worse functions from its dual will be at approximating functions in $\cF^*$. In other words, the simpler $\cG$ is, the worse the approximation $\frac{1}{N} \sum_{x \in \sample}  \norm{f_{x}^* - g_{x}^*}_{\infty}$ will likely be. Therefore, there is a tradeoff between generalizability and approximability. It may not be \emph{a priori} clear how to balance this tradeoff. \emph{Structural risk minimization (SRM)} is a classic, well-studied approach for optimizing tradeoffs between complexity and generalizability which we use in our experiments.

Our SRM approach is based on the following corollary of Theorem~\ref{thm:rad}. Let $\cG_1, \cG_2, \cG_3, \dots$ be a countable sequence of function classes where each $\cG_j = \left\{g_{j,\vec{r}} \mid \vec{r} \in \cR\right\}$ is a set of functions mapping $\cX$ to $[0,1]$. We use the notation $g_{j, x}^*$ to denote the duals of the functions in $\cG_j$, so $g_{j, x}^*(\vec{r}) = g_{j, \vec{r}}(x)$.
\begin{restatable}{cor}{corSRM}\label{cor:SRM}
With probability $1-\delta$ over the draw of the set $\sample \sim \dist^N$, for all $\vec{r} \in \cR$ and all $j \geq 1$, \begin{equation}\left|\frac{1}{N} \sum_{x \in \sample} f_{\vec{r}}(x) - \E_{x \sim \dist} \left[f_{\vec{r}}(x)\right]\right|	= O\left(\frac{1}{N} \sum_{x \in \sample}  \norm{f_{x}^* - g_{j, x}^*}_{\infty} + \erad\left(\cG_j\right) + \sqrt{\frac{1}{N} \ln \frac{j}{\delta}}\right).\label{eq:SRM}\end{equation}
\end{restatable}
The proof of this corollary is in Appendix~\ref{app:UC}.

In our experiments,
each dual class $\cG^*_j$ consists of piecewise-constant functions with at most $j$ pieces. This means that as $j$ grows, the class $\cG_j^*$ becomes more complex, or in other words, the Rademacher complexity $\erad\left(\cG_j\right)$ also grows. Meanwhile, the more pieces a piecewise-constant function $g_x^*$ has, the better it is able to approximate the dual function $f_x^*$. In other words, as $j$ grows, the approximation term $ \frac{1}{N} \sum_{x \in \sample}  \norm{f_{x}^* - g_{j, x}^*}_{\infty}$ shrinks. SRM is the process of finding the level $j$ in the nested hierarchy that minimizes the sum of these two terms, and therefore obtains the best generalization guarantee via Equation~\eqref{eq:SRM}.

\begin{remark}\label{remark:exp}
	We conclude by noting that the empirical average $ \frac{1}{N} \sum_{x \in \sample}  \norm{f_{x}^* - g_{j, x}^*}_{\infty}$ in Equation~\eqref{eq:SRM} can be replaced by the expectation $\E_{x \sim \dist}\left[\norm{f_x^* - g_{j,x}^*}_{\infty}\right]$. See Corollary~\ref{cor:SRM_exp} in Appendix~\ref{app:UC} for the proof.
	\end{remark}

%% file: IP.tex
In this section, we demonstrate that our data-dependent generalization guarantees from Section~\ref{sec:UC} can be much tighter than worst-case generalization guarantees provided in prior research. We demonstrate these improvements in the context of integer programming algorithm configuration, which we introduced in Section~\ref{sec:IP_intro}. 
 Our formal model is the same as that of \citet{Balcan18:Learning}, who studied worst-case generalization guarantees.
Each element of the set $\cX$ is an IP. The set $\cR$ consists of CPLEX parameter settings. We assume there is an upper bound $\kappa$ on the size of the largest tree we allow B\&B to build before we terminate, as in prior research~\citep{Hutter09:ParamILS, Kleinberg17:Efficiency, Balcan18:Learning, Kleinberg19:Procrastinating}. In Appendix~\ref{app:IP}, we describe our methodology for choosing $\kappa$. Given a parameter setting $r$ and an IP $x$, we define $f_r(x)$ to be the size of the tree CPLEX builds, capped at $\kappa$, divided by $\kappa$ (this way, $f_r(x) \in [0,1]$). We define the set $\cF = \left\{f_r \mid r \in \cR\right\}$.
 
We tune the parameter of B\&B's variable selection policy (VSP). We described the purpose of VSPs in Section~\ref{sec:IP_intro}. We study \emph{score-based VSPs}, defined as follows.
	Let $\score$ be a function that takes as input a partial B\&B tree $\tree$, a leaf of $\tree$ representing an IP $x$, and an index $i \in [n]$, and returns a real-valued $\score(\tree, x, i)$. Let $V$ be the set of variables that have not been branched on along the path from the root of $\tree$ to $x$. A \emph{score-based VSP} branches on the variable $\argmax_{z_i \in V} \{\score(\tree, x, i)\}$ at the node $x$.

We study how to learn a high-performing convex combination of any two scoring rules. We focus on four scoring rules in our experiments.
 To define them, we first introduce some notation. For an IP $x$ with objective function $\vec{c} \cdot \vec{z}$, we denote an optimal solution to the LP relaxation of $x$ as $\breve{\vec{z}}_x = \left(\breve{z}_{x,1}, \dots \breve{z}_{x,n}\right)$. We also use the notation $\breve{c}_x = \vec{c} \cdot \breve{\vec{z}}_x$. Finally, we use the notation $x_i^+$ (resp., $x_i^-$) to denote the IP $x$ with the additional constraint that $z_i = 1$ (resp., $z_i = 0$).\footnote{If $x_i^+$ (resp., $x_i^-$) is infeasible, then we define $\breve{c}_x - \breve{c}_{x_i^+}$ (resp., $\breve{c}_x - \breve{c}_{x_i^-}$) to be some large number greater than $||\vec{c}||_1$.}
 
We study four scoring rules $\score_L$, $\score_S$, $\score_A$, and $\score_P$:
\begin{itemize}
	\item $\score_L(\tree, x, i) = \max\left\{\breve{c}_x - \breve{c}_{x_i^+}, \breve{c}_x - \breve{c}_{x_i^-}\right\}$. Under $\score_L$, B\&B branches on the variable leading to the \textbf{L}argest change in the LP objective value.
	\item $\score_S(\tree, x, i) = \min\left\{\breve{c}_x - \breve{c}_{x_i^+}, \breve{c}_x - \breve{c}_{x_i^-}\right\}$. Under $\score_S$, B\&B branches on the variable leading to the \textbf{S}mallest change.
	\item $\score_A(\tree, x, i) = \frac{1}{6}\score_L(\tree, x, i) + \frac{5}{6} \score_S(\tree, x, i)$. This is a scoring rule that \citet{Achterberg09:SCIP} recommended. It balances the optimistic approach to branching under $\score_L$ with the pessimistic approach under $\score_S$.
	\item $\score_P(\tree, x, i) = \max\left\{\breve{c}_x - \breve{c}_{x_i^+}, 10^{-6}\right\}  \cdot \max\left\{\breve{c}_x - \breve{c}_{x_i^-}, 10^{-6}\right\}$. This is known as the \emph{\textbf{P}roduct scoring rule}. Comparing $\breve{c}_x - \breve{c}_{x_i^-}$ and $\breve{c}_x - \breve{c}_{x_i^+}$ to $10^{-6}$ allows the algorithm to compare two variables even if $\breve{c}_x - \breve{c}_{x_i^-} = 0$ or $\breve{c}_x - \breve{c}_{x_i^+} = 0$. After all, suppose the scoring rule simply calculated the product $\left(\breve{c}_x - \breve{c}_{x_i^-}\right)\cdot \left(\breve{c}_x - \breve{c}_{x_i^+}\right)$ without comparing to $10^{-6}$. If $\breve{c}_x - \breve{c}_{x_i^-} = 0$, then the score equals 0, canceling out the value of $\breve{c}_x - \breve{c}_{x_i^+}$ and thus losing the information encoded by this difference.
\end{itemize}
 
Fix any two scoring rules $\score_1$ and $\score_2$. We define $f_r(x)$ to be the size of the tree B\&B builds (normalized to lie in $[0,1]$) when it uses the score-based VSP defined by $(1-r)\score_1 + r\score_2$. Our goal is to learn the best convex combination of the two scoring rules. When $\score_1 = \score_L$ and $\score_2 = \score_S$, prior research has proposed several alternative settings for the parameter $r$~\citep{Gauthier77:Experiments, Benichou71:Experiments, Beale79:Branch,Linderoth99:Computational,Achterberg09:SCIP}, though no one setting is optimal across all applications.
 \citet{Balcan18:Learning} prove the following lemma about the structure of the functions $f_x^*$.
 \begin{lemma}\label{lem:WC_IP}
 	For any IP $x$ with $n$ variables, the dual function $f_x^*$ is piecewise-constant with at most $n^{2(\kappa + 1)}$ pieces.
 \end{lemma}

Lemma~\ref{lem:WC_IP} implies the following worst-case bound on $\erad\left(\cF\right)$. See Lemma~\ref{lem:PWC} in Appendix~\ref{app:IP} for the proof.
\begin{cor}\label{cor:WC_UC}
For any set $\sample \subseteq \cX$ of integer programs, $\erad\left(\cF\right) \leq \sqrt{\frac{2 \ln\left(|\sample|\left(n^{2(\kappa + 1)}-1\right) + 1\right)}{|\sample|}}.$
	\end{cor}

This corollary and Theorem~\ref{thm:rad_gen} imply the following worst-case generalization bound: with probability $1-\delta$ over the draw of $N$ samples $\sample \sim \dist^N$, for all $r \in [0,1]$, $\left|\frac{1}{N} \sum_{x \in \sample} f_{r}(x) - \E_{x \sim \dist} \left[f_{r}(x)\right]\right|$ is bounded above by \begin{equation}2\sqrt{\frac{2 \ln(N\left(n^{2(\kappa + 1)}-1\right) + 1)}{N}} + 3\sqrt{\frac{1}{2N}\ln \frac{2}{\delta}}.\label{eq:WC_gen}\end{equation}
This worst-case bound can be large when $\kappa$ is large. We find that although the duals $f_x^*$ are piecewise-constant with many pieces, they can be approximated piecewise-constant functions with few pieces, as in Figure~\ref{fig:apx_pwc_pwc}. As a result, we improve over Equation~\eqref{eq:WC_gen} via Theorem~\ref{thm:rad}, our data-dependent bound.

To make use of Theorem~\ref{thm:rad}, we now formally define the function class whose dual $(\gamma, \infty)$-approximates $\cF^*$. We first define the dual class, then the primal class. To this end, fix some integer $j \geq 1$ and let $\cH_j$ be the set of all piecewise-constant functions mapping $[0,1]$ to $[0,1]$ with at most $j$ pieces. For every IP $x$, we define $g_{j,x}^* \in \argmin_{h \in \cH_j} \norm{f_x^* - h}_{\infty}$, breaking ties in some fixed but arbitrary manner. We define the dual class $\cG^*_j = \left\{g_{j,x}^* \mid x \in \cX\right\}$. Therefore, the dual class $\cG^*_j$ is consists of piecewise-constant functions with at most $j$ pieces. In keeping with the definition of primal and dual functions from Section~\ref{sec:dual}, for every parameter $r \in [0,1]$ and IP $x$, we define $g_{j,r}(x) = g_{j,x}^*(r)$. Finally, we define the primal class $\cG_j = \left\{g_{j,r} \mid r \in [0,1]\right\}.$

To apply our results from Section~\ref{sec:UC}, we must bound the Rademacher complexity of the set $\cG_j$. Doing so is simple due to the structure of the dual class $\cG^*_j$. The following lemma\footnote{This bound on $\erad\left(\cG_j\right)$ could potentially be optimized even further using a data-dependent approach, such as the one summarized by Theorem E.3 in the paper by \citet{Balcan18:Learning}.} is a corollary of Lemma~\ref{lem:PWC} in Appendix~\ref{app:IP}.
\begin{lemma}\label{lem:IP_PWC}
For any set $\sample \subseteq \cX$ of integer programs, $\erad\left(\cG_j\right) \leq \sqrt{\frac{2 \ln(|\sample|(j-1) + 1)}{|\sample|}}.$
	\end{lemma}
This lemma together with Remark~\ref{remark:exp} and Corollary~\ref{cor:SRM} imply that with probability $1-\delta$ over $\sample  \sim \dist^N$, for all parameters $r \in [0,1]$ and $j \geq 1$, $\left|\frac{1}{N} \sum_{x \in \sample} f_{r}(x) - \E_{x \sim \dist} \left[f_{r}(x)\right]\right|$ is upper-bounded by the minimum of Equation~\eqref{eq:WC_gen} and
\begin{equation}2\E_{x \sim \dist}\left[\norm{f_x^* - g_{j,x}^*}_{\infty}\right] + 2\sqrt{\frac{2 \ln(N(j-1) + 1)}{N}} + \sqrt{\frac{2}{N}\ln \frac{2(\pi j)^2}{3\delta}}.\label{eq:DD_gen}\end{equation}
As $j$ grows, $\erad\left(\cG_j\right)$ grows, but the dual class $\cG^*_j$ is better able to approximate $\cF^*$. In our experiments, we optimize this tradeoff between generalizability and approximability.

\paragraph{Experiments.} We analyze distributions over IPs formulating the combinatorial auction winner determination problem under the OR-bidding language~\citep{Sandholm02:Algorithm}, which we generate using the Combinatorial
Auction Test Suite (CATS)~\citep{Leyton00:Toward}.
We use the ``arbitrary'' generator with 200 bids and 100 goods, resulting in IPs with 200 about variables, and the ``regions'' generator with 400
bids and 200 goods, resulting in IPs with 400 about variables.

We use the algorithm described in Appendix D.1 of the paper by \citet{Balcan18:Learning} to compute the functions $f_x^*$. It overrides the default
VSP of CPLEX 12.8.0.0 using the C API. We use Algorithm~\ref{alg:DP} in Appendix~\ref{app:IP} to compute the approximating duals. All experiments were run on a 64-core machine with 512 GB of RAM.
\begin{figure}[t]
	\centering
	\subfigure[Results on the CATS ``regions'' generator with $\score_1 = \score_L$ and $\score_2 = \score_S$.\label{fig:regions}]{\includegraphics[width=0.47\textwidth]{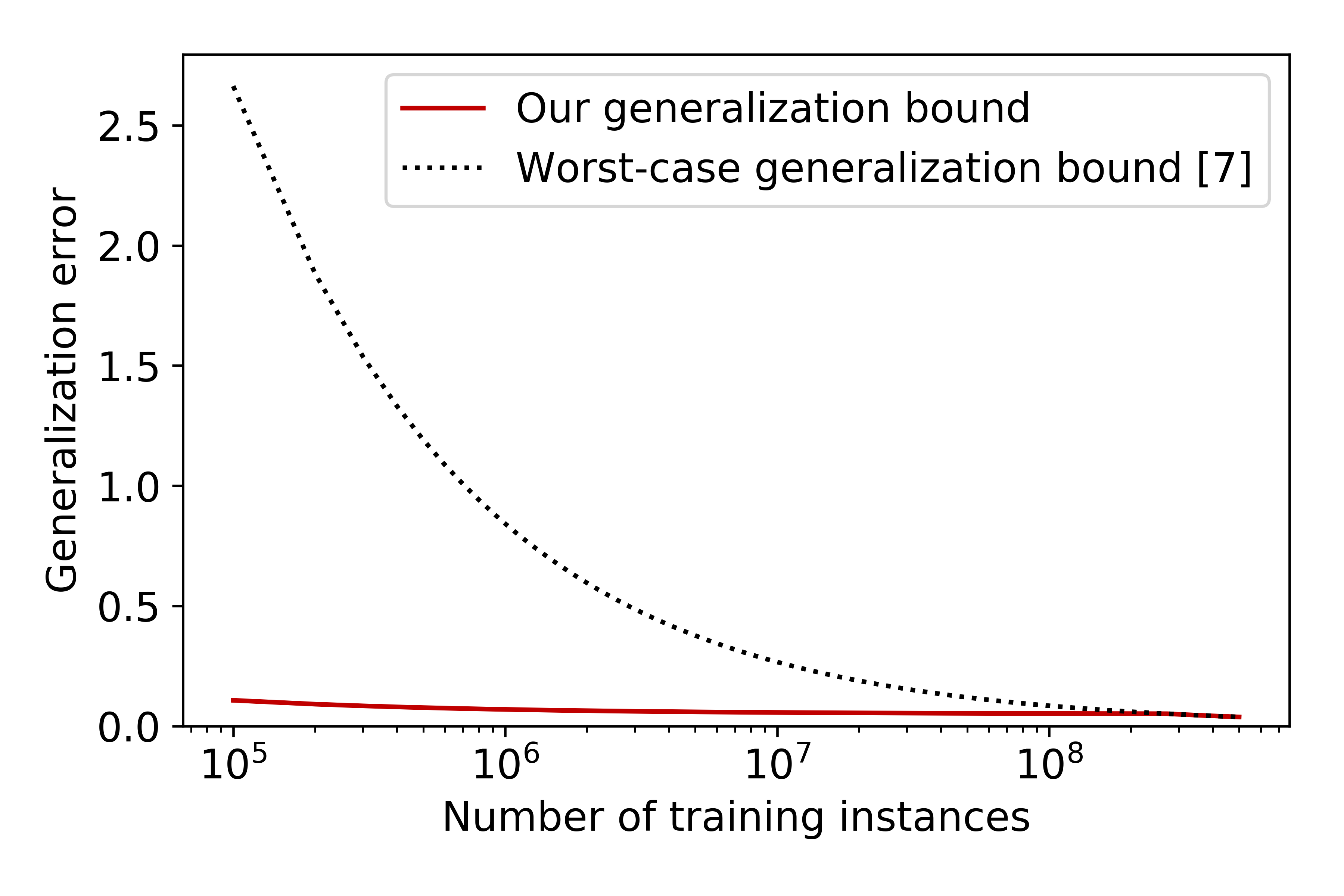}}\qquad
	\subfigure[Results on the CATS ``arbitrary'' generator with $\score_1 = \score_L$ and $\score_2 = \score_S$.\label{fig:arbitrary}]{\includegraphics[width=0.47\textwidth]{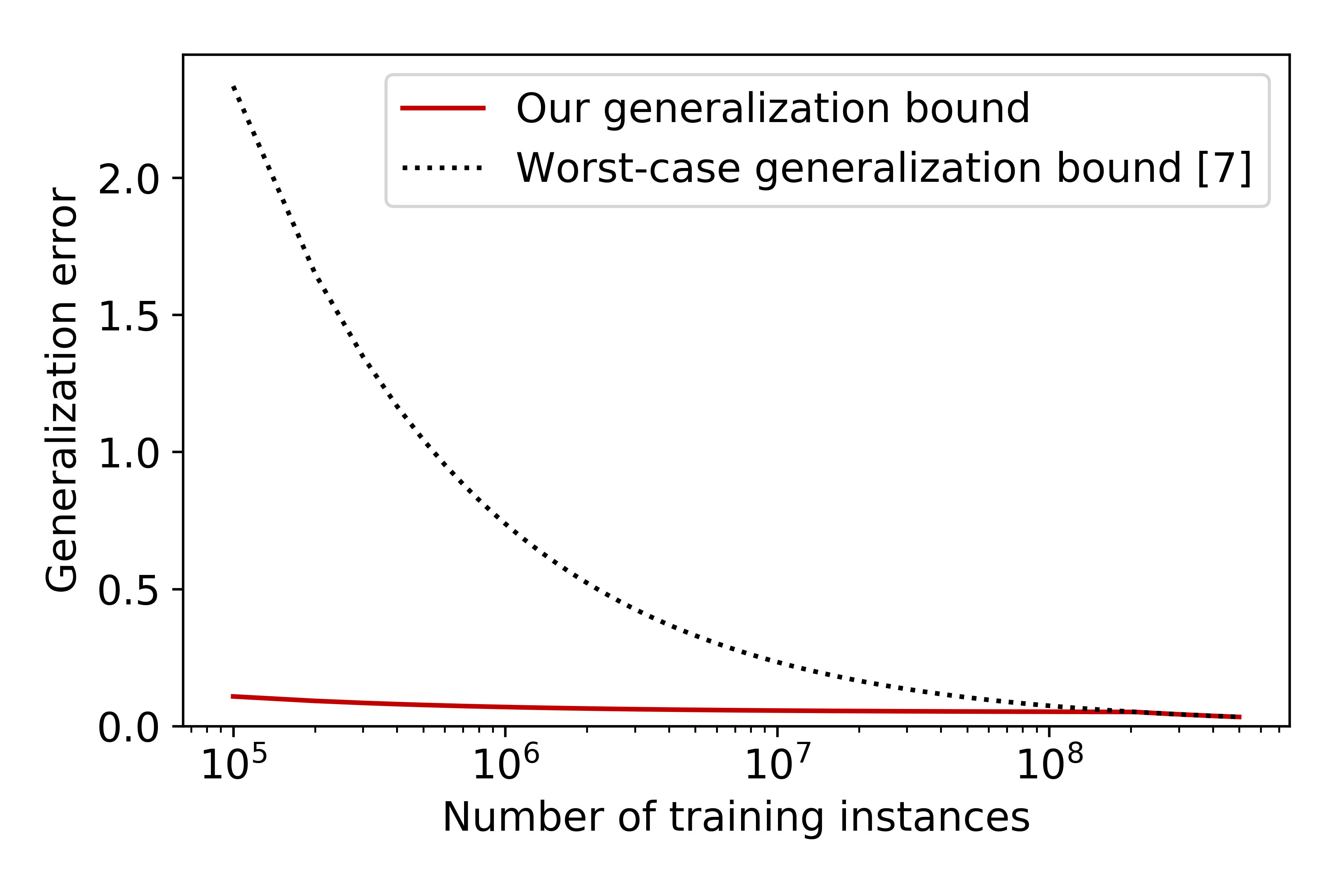}}
	\subfigure[{Results on the CATS ``arbitrary'' generator with $\score_1 = \score_P$ and $\score_2 = \score_A$.\label{fig:arbitrary_product}}]{\includegraphics[width=0.47\textwidth]{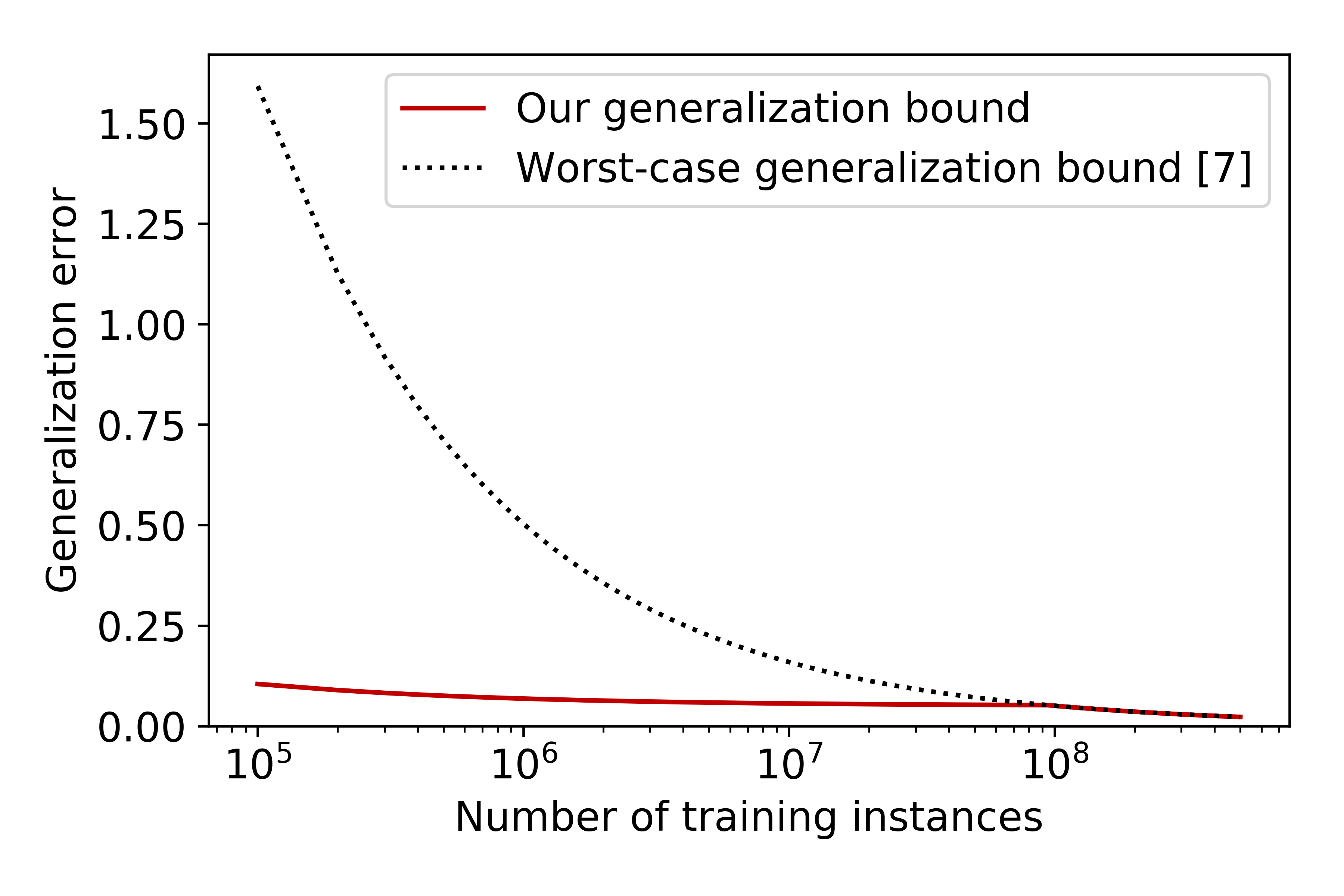}}
	\caption{Experiments where we compare our generalization bound to the worst-case bound by~\citet{Balcan18:Learning}. The red solid line is our generalization bound: the minimum of Equations~\eqref{eq:WC_gen} and \eqref{eq:estimate_gen} as a function of the number of training examples $N$. The black dotted line is the worst-case bound from Equation~\eqref{eq:WC_gen}.}
	\label{fig:experiments}
\end{figure}

In Figure~\ref{fig:experiments}, we select $\score_1, \score_2 \in \left\{\score_L, \score_S, \score_A, \score_P\right\}$ and compare the worst-case and data-dependent bounds.
First, we plot the worst-case bound from Equation~\eqref{eq:WC_gen}, with $\delta = 0.01$, as a function of the number of training examples $N$. This is the black, dotted line in Figure~\ref{fig:experiments}.

Next, we plot the data-dependent bound, which is the red, solid line in Figure~\ref{fig:experiments}. To calculate the data-dependent bound in Equation~\eqref{eq:DD_gen}, we have to estimate $\E_{x \sim \dist}\left[\norm{f_x^* - g_{j,x}^*}_{\infty}\right]$ for all $j \in [1600]$.\footnote{We choose the range $j \in [1600]$ because under these distributions, the functions $f_x^*$ generally have at most 1600 pieces.}
To do so, we draw $M = 6000$ IPs
$x_1, \dots, x_M$ from the distribution $\dist$. We estimate $\E_{x \sim \dist}\left[\norm{f_x^* - g_{j,x}^*}_{\infty}\right]$ via the empirical average $\frac{1}{M} \sum_{i = 1}^M \norm{f_{x_i}^* - g_{j,x_i}^*}_{\infty}$. A Hoeffding bound guarantees that with probability 0.995, for all $j \in [1600]$, \begin{equation}\E\left[\norm{f_x^* - g_{j,x}^*}_{\infty}\right]
\leq\frac{1}{M} \sum_{i = 1}^M \norm{f_{x_i}^* - g_{j,x_i}^*}_{\infty} + \frac{1}{40}.\label{eq:estimate}\end{equation} We prove this inequality in Lemma~\ref{lem:estimate}.
We thereby estimate our data-dependent bound Equation~\eqref{eq:DD_gen} using the following bound:
	\begin{equation}\min_{j \in [1600]}\left\{2\left(\frac{1}{M} \sum_{i = 1}^M \norm{f_{x_i}^* - g_{j,x_i}^*}_{\infty} + \frac{1}{40}\right) + 2\sqrt{\frac{2 \ln(N(j-1) + 1)}{N}} + \sqrt{\frac{2}{N}\ln \frac{(20\pi j)^2}{3}}\right\}.\label{eq:estimate_gen}\end{equation}
The only difference between Equations~\eqref{eq:DD_gen} and \eqref{eq:estimate_gen} is that Equation~\eqref{eq:DD_gen} relies on the left-hand-side of Equation~\eqref{eq:estimate} and Equation~\eqref{eq:estimate_gen} relies on the right-hand-side of Equation~\eqref{eq:estimate} and sets $\delta = 0.005$.\footnote{Like the worst-case bound, Equation~\eqref{eq:estimate_gen} holds with probability 0.99, because with probability 0.995, Equation~\eqref{eq:estimate} holds, and with probability 0.995, the bound from Equation~\eqref{eq:DD_gen} holds.}
 In Figure~\ref{fig:experiments}, the red solid line equals the minimum of Equations~\eqref{eq:WC_gen} and \eqref{eq:estimate_gen} as a function of the number of training examples $N$. 

In Figure~\ref{fig:experiments}, we see that our bound significantly beats the worst-case bound up until the point there are approximately 100,000,000 training instances. At this point, the worst-case guarantee is better than the data-dependent bound, which makes sense because it goes to zero as $N$ goes to infinity, whereas the term $\frac{1}{M} \sum_{i = 1}^M \norm{f_{x_i}^* - g_{j,x_i}^*}_{\infty} +\frac{1}{40}$ in our bound (Equation~\eqref{eq:estimate_gen}) is a constant.

Figure~\ref{fig:regions} also illustrates that even when there are only $10^5$ training instances, our bound provides a generalization guarantee of approximately 0.1. Meanwhile, $7 \cdot 10^7$ training instances are necessary to provide a generalization guarantee of 0.1 under the worst-case bound, so the sample complexity implied by our analysis is 700 times better. Similarly, in Figure~\ref{fig:arbitrary}, 
500 times fewer samples are required to obtain a generalization guarantee of 0.1 under our bound versus the worst-case bound. In Figure~\ref{fig:arbitrary_product}, 250 times fewer samples are required.
 
In this section, we approximated the dual functions $f_x^*$ with piecewise constant functions that have a small number of pieces --- say, $j$ pieces. We used SRM to find the value for $j$ which leads to the strongest bounds, as in Equation~\eqref{eq:estimate_gen}. In Appendix~\ref{app:experiments}, we compare against another baseline where we do not use SRM, but simply set $j$ to be the maximum number of pieces we observe over our training set. Of course, this bound is much tighter than the worst-case bound by \citet{Balcan18:Learning}, the baseline in  Figure~\ref{fig:experiments}. However, we still observe that for a target generalization error, the number of samples required according to our bound is up to four times smaller than the number of samples required by this baseline.

%% file: not_learnable.tex
In this section, we show that $(\gamma, p)$-approximability with $p < \infty$ does not necessarily imply strong generalization guarantees of the type we saw in Section~\ref{sec:UC}. We show that it is possible for a dual class $\cF^*$ to be well-approximated by the dual of a class $\cG$ with $\erad(\cG)  = 0$, yet for the primal $\cF$ to have high Rademacher complexity.

\begin{figure}\centering
	\includegraphics{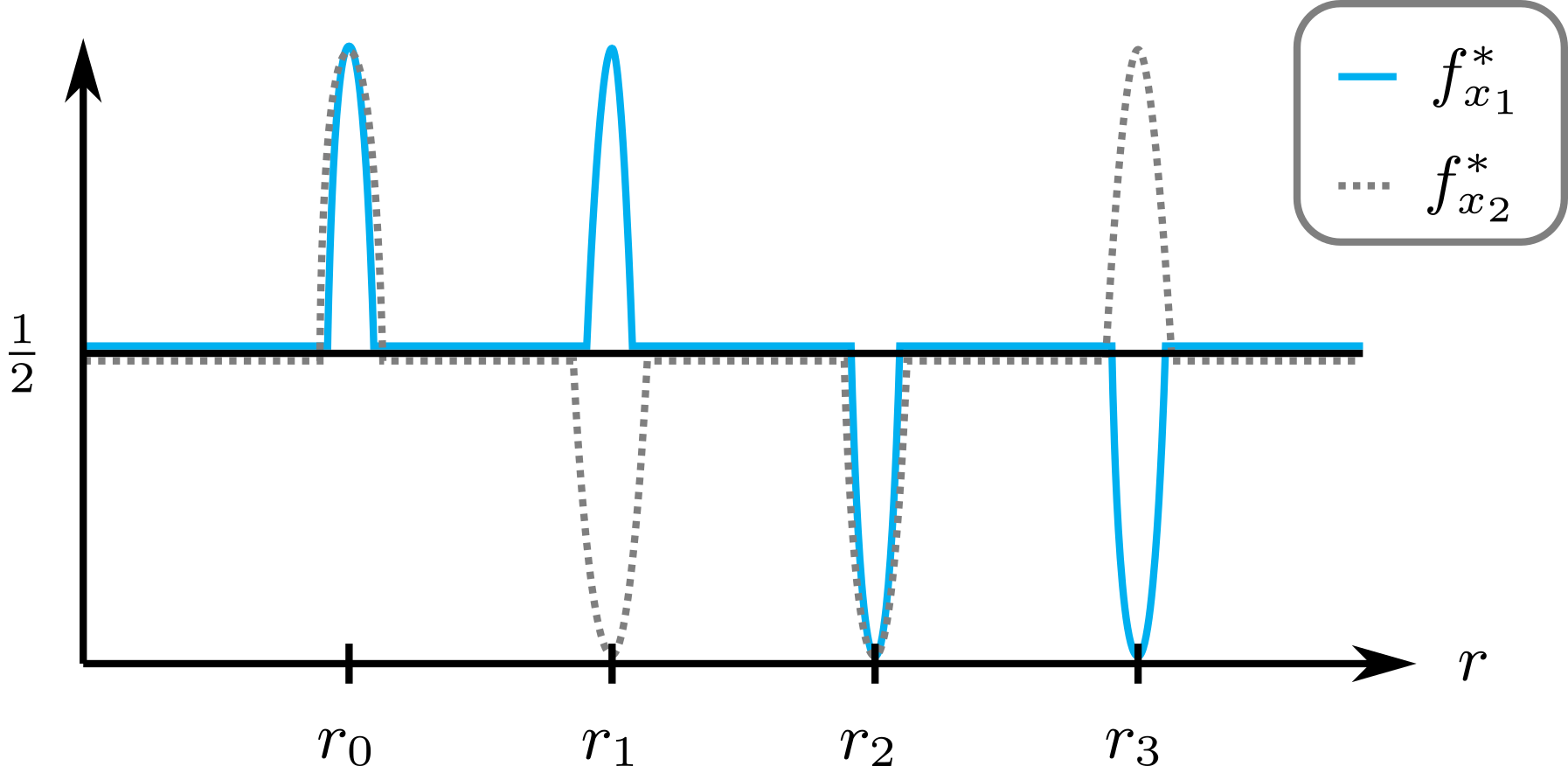}
	\caption{The dual functions $f_{x_1}^*$ and $f_{x_2}^*$ are well-approximated by the constant function $r \mapsto \frac{1}{2}$ under, for example, the $L^1$-norm because the integrals $\int_{\cR}\left|f_{x_i}^*(r) - \frac{1}{2}\right| \, dr$ are small; for most $r$, $f_{x_i}^*(r) = \frac{1}{2}$. The approximation is not strong under the $L^{\infty}$-norm, since $\max_{r \in \cR}\left|f_{x_i}^*(r) - \frac{1}{2}\right| = \frac{1}{2}$. The function class $\cF$ corresponding to these duals has a large Rademacher complexity.\label{fig:l_p}}
\end{figure}
\begin{figure}
	\includegraphics{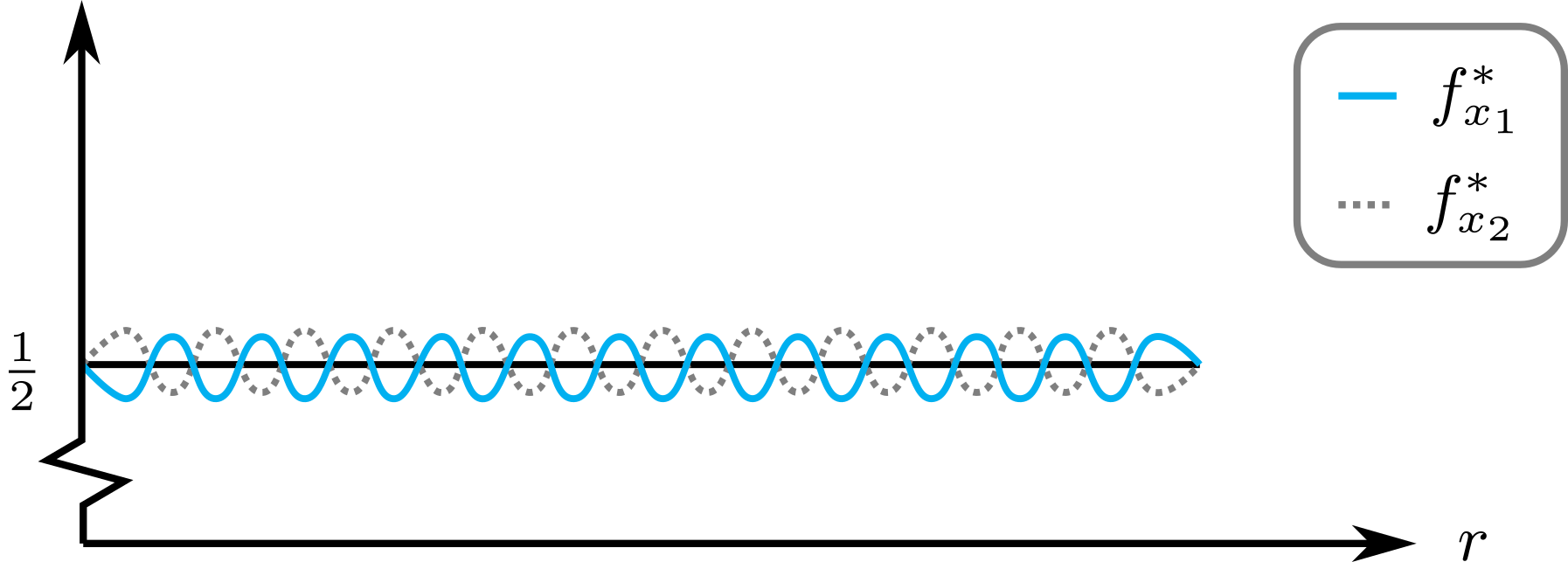}\centering
	\caption{The dual functions $f_{x_1}^*$ and $f_{x_2}^*$ are well-approximated by the constant function $r \mapsto \frac{1}{2}$ under the  $L^{\infty}$-norm since $\max_{r \in \cR}\left|f_{x_i}^*(r) - \frac{1}{2}\right|$ is small. The function class $\cF$ corresponding to these duals has a small Rademacher complexity.\label{fig:l_infty}}
\end{figure}
Figures~\ref{fig:l_p} and \ref{fig:l_infty} help explain why there is this sharp constrast between the $L^{\infty}$- and $L^p$-norms for $p < \infty$. Figure~\ref{fig:l_p} illustrates two dual functions $f_{x_1}^*$ (the blue solid line) and $f_{x_2}^*$ (the grey dotted line). Let $\cG$ be the extremely simple function class $\cG = \left\{g_r : r \in \R\right\}$ where $g_r(x) = \frac{1}{2}$ for every $x \in \cX$. It is easy to see that $\erad(\cG) = 0$ for any set $\sample$. Moreover, every dual function $g_x^*$ is also simple, because $g_x^*(r) = g_r(x) = \frac{1}{2}$. From Figure~\ref{fig:l_p}, we can see that the functions $f_{x_1}^*$ and $f_{x_2}^*$ are well approximated by the constant function $g_{x_1}^*(r) = g_{x_2}^*(r) = \frac{1}{2}$ under, for example, the $L^1$-norm because the integrals $\int_{\cR}\left|f_{x_i}^*(r) - \frac{1}{2}\right| \, dr$ are small. However, the approximation is not strong under the $L^{\infty}$-norm, since $\max_{r \in \cR}\left|f_{x_i}^*(r) - \frac{1}{2}\right| = \frac{1}{2}$ for $i \in \{1,2\}$.

Moreover, despite the fact that $\erad(\cG) = 0$, we have that $\erad(\cF) = \frac{1}{2}$ when $\sample = \left\{x_1, x_2\right\}$, which makes Theorem~\ref{thm:rad_gen} meaningless. At a high level, this is because when $\sigma_1 = 1$, we can ensure that $\sigma_1 f_r\left(x_1\right) = \sigma_1 f_{x_1}^*\left(r\right) = 1$ by choosing $r \in \left\{r_0, r_1\right\}$ and when $\sigma_1 = -1$, we can ensure that $\sigma_1 f_r\left(x_1\right) = 0$ by choosing $r \in \left\{r_2, r_3\right\}$. A similar argument holds for $\sigma_2$. In summary, $(\gamma, p)$-approximability for $p < \infty$ does not guarantee low Rademacher complexity.

Meanwhile, in Figure~\ref{fig:l_infty}, $g_{x_i}^*(r) = \frac{1}{2}$ and $f_{x_i}^*(r)$ are close for every parameter $r$. As a result, for any noise vector $\vec{\sigma}$, $\sup_{r \in \R} \left\{\sigma_1 f_{x_1}^*(r) + \sigma_2f_{x_2}^*(r)\right\}$ is close to $\sup_{r \in \R} \left\{\sigma_1 g_{x_1}^*(r) + \sigma_2g_{x_2}^*(r)\right\}$. This implies that the Rademacher complexities $\erad(\cG)$ and $\erad(\cF)$ are close.
This illustration exemplifies Theorem~\ref{thm:rad}: $(\gamma, \infty)$-approximability implies strong Rademacher bounds.

We now prove that $(\gamma, p)$-approximability by a simple class for $p < \infty$ does not guarantee low Rademacher complexity.
\begin{theorem}\label{thm:not_learnable}
For any $\gamma \in (0,1/4)$ and any $p \in [1, \infty)$, there exist function classes $\cF, \cG \subset [0,1]^{\cX}$ such that the dual class $\cG^*$ $(\gamma, p)$-approximates $\cF^*$ and for any $N \geq 1$, $\sup_{\sample: |\sample| = N}\erad(\cG) = 0$ and $\sup_{\sample: |\sample| = N}\erad(\cF) = \frac{1}{2}.$
\end{theorem}

\begin{proof} We begin by defining the classes $\cF$ and $\cG$. Let  $\cR = \left(0,\gamma^p\right]$, and $\cX = \left[\gamma^{-p} / 2, \infty\right)$. For any $r \in \cR$ and $x \in \cX$, let $f_r(x) = \frac{1}{2}(1 + \cos(rx))$ and $\cF = \left\{ f_r \mid r \in \cR\right\}$. These sinusoidal functions are based on the intuition from Figure~\ref{fig:l_p}. As in Figure~\ref{fig:l_p}, for any $r$ and $x$, let $g_r(x) = \frac{1}{2}$ and $\cG = \left\{g_{r} \mid r \in \cR\right\}$.
	Since $\cG$ consists of identical copies of a single function, $\erad(\cG) = 0$ for any set $\sample \subseteq \cX$. Meanwhile, in Lemma~\ref{lem:not_learnable} in Appendix~\ref{app:not_learnable}, we prove that for any $N \geq 1$, $\sup_{\sample: |\sample| = N}\erad(\cF) = \frac{1}{2}.$
	
	In Lemma~\ref{lem:approximates} in Appendix~\ref{app:not_learnable}, we prove that the dual class $\cG^*$ $(\gamma, p)$-approximates $\cF^*$. To prove this, we first show that $\norm{f_x^* - g_x^*}_2 \leq \frac{1}{4} \sqrt{2 \gamma^p + \frac{1}{x}}$.
	When $p = 2$, we know $\frac{1}{x} \leq 2\gamma^2$, so $\norm{f_x^* - g_x^*}_2 < \gamma$. Otherwise, we use our bound on $\norm{f_x^* - g_x^*}_2$, H\"older's inequality, and the log-convexity of the $L^p$-norm to prove that $\norm{f_x^* - g_x^*}_p \leq \gamma$.
\end{proof}
\begin{remark}\label{remark:uniform}
	Suppose, for example, that $\cR = [0,1]^d$. 
Theorem~\ref{thm:not_learnable} implies that even if \[\left|f_x^*(\vec{r}) - g_x^*(\vec{r})\right|\] is small for all $x$ in expectation over $\vec{r} \sim \text{Uniform}(\cR)$, the function class $\cF$ may not have Rademacher complexity close to $\cG$.
\end{remark}

\paragraph{Statistical learnability.} In Appendix~\ref{app:statistical}, we connect our results to the literature on statistical learnability~\citep{Haussler92:Decision}. At a high level, a function class $\cF$ is \emph{statistically learnable} (Definition~\ref{def:statistical} in Appendix~\ref{app:notation}) if there exists a learning algorithm that returns a function whose expected value converges---as the size of the training set grows---to the smallest expected value of any function in $\cF$. We introduce a relaxation: a function class $\cF$ is $\gamma$-\emph{statistically learnable} (Definition~\ref{def:gamma}) if, at a high level, there exists a learning algorithm with error at most $\gamma$ in the limit as the training set size grows. We prove that if the dual class $\cF^*$ is $(\gamma, \infty)$-approximated by the dual of a statistically learnable class $\cG$, then $\cF$ is $\gamma$-statistically learnable. On the other hand, Theorem~\ref{thm:not_learnable} implies that there exists a class $\cF$ that is not $\gamma$-statistically learnable, yet it is $(\gamma, p)$-approximated by the dual of a statistically learnable class $\cG$.

%% file: conclusion.tex
We provided generalization guarantees for algorithm configuration, which bound the difference between a parameterized algorithm's average empirical performance over a set of sample problem instances and its expected performance on future, unseen instances. We did so by exploiting structure exhibited by the \emph{dual functions} which measure the algorithm's performance as a function of its parameters. We analyzed the widely-applicable setting where the dual functions are approximated by ``simple'' functions. We showed that if this approximation holds under the $L^{\infty}$-norm, then it is possible to provide strong generalization guarantees. If, however, the approximation only holds under the $L^p$-norm for $p <\infty$, we showed that it is impossible in the worst-case to provide non-trivial bounds.
Via experiments in the context of integer programming algorithm configuration, we demonstrated that our bounds can be significantly stronger than the best-known worst-case guarantees~\citep{Balcan18:Learning}, leading to a sample complexity improvement of 70,000\%.

We conclude with a direction for future research. Suppose, for some prior $\cP$ over parameters, $\E_{x \sim \dist, \vec{r} \sim \cP}\left[\left|f_x^*(\vec{r}) - g_x^*(\vec{r})\right|\right]$ is small. From Remark~\ref{remark:uniform}, we know strong generalization bounds are not possible in the worst case, but what about under some realistic assumptions? This may help us understand, for example, why random forests---which have a simple piecewise-constant structure---are often able to accurately predict the runtime of SAT and MIP solvers~\citep{Hutter11:Sequential}.

\paragraph{Acknowledgments.}
We thank Kevin Leyton-Brown for a stimulating discussion that inspired us to pursue this research direction.

 This material is based on work supported by the National Science Foundation under grants CCF-1535967, CCF-1733556, CCF-1910321, IIS-1617590, IIS-1618714, IIS-1718457, IIS-1901403, and SES-1919453; the ARO under awards W911NF-17-1-0082 and W911NF2010081; a fellowship from Carnegie Mellon University’s Center for Machine Learning and Health; the Defense Advanced Research Projects Agency under cooperative agreement HR00112020003; an Amazon Research Award; an AWS Machine Learning Research Award; an Amazon Research Award; a Bloomberg Research Grant; and a Microsoft Research Faculty Fellowship.

%% file: learning_theory.tex
 In this appendix, we study a more general setting than in the main body where the learning algorithms have access to examples $x \in \cX$ that may be labeled by a real value $y \in \R$. The learning algorithms we analyze have sample access to an unknown distribution $\dist$ over (labeled) examples $(x,y) \in \cX \times [0,1]$. The fact that the examples are labeled is without loss of generality; in our integer programming algorithm configuration example, there are no labels, or equivalently, for every tuple $(x,y)$ in the support of $\dist$, $y = 0$. We use the notation $\dist|_{\cX}$ to denote the marginal distribution of $\dist$ over $\cX$.
 
 Given a set of functions $\cF \subseteq [0,1]^{\cX}$, the learning algorithms we study aim to find a function $h : \cX \to [0,1]$ with expected absolute loss $\E_{(x,y) \sim \dist}[|h(x)-y|]$ that is nearly as small as the smallest expected loss of any function in $\cF$, $\inf_{f \in \cF}\E_{(x,y) \sim \dist}[|f(x)-y|]$. (Though we focus on absolute loss in this paper, we believe our results can be generalized to other loss functions, which we leave for future research.) The function $h$ may or may not be a member of the set $\cF$, depending on the specific learning task at hand.
 
 In the integer programming example from the main body, the functions in $\cF$ are parameterized by CPLEX parameter vectors $\vec{r} \in \R^d$, denoted $\cF = \left\{f_{\vec{r}} \mid \vec{r} \in \R^d\right\}.$ As we described in the main body, for any integer program $x \in \cX$ and parameter vector $\vec{r} \in \R^d$, $f_{\vec{r}}(x)$ equals the size of the branch-and-bound tree CPLEX builds given the parameter setting $\vec{r}$ and input IP $x$, normalized to fall within the interval $[0,1]$. The learning algorithms we study take as input a training set of integer programs $x_1, \dots, x_N \sim \dist|_{\cX}$ and return a CPLEX parameter vector $\hat{\vec{r}} \in \R^d$, or equivalently, a function $f_{\hat{\vec{r}}} \in \cF$. Since our goal is to minimize tree size, ideally, the size of the trees CPLEX builds using the parameter setting $\hat{\vec{r}}$ should be small in expectation over $\dist$ when compared with the best choice of a parameter setting. In other words, $\E_{(x,y) \sim \dist}\left[f_{\hat{\vec{r}}}(x)\right] - \inf_{f_{\vec{r}} \in \cF}\E_{(x,y) \sim \dist}\left[f_{\vec{r}}(x)\right]$ should be small. (Recall that in this setting, for every tuple $(x,y)$ in the support of $\dist$, $y = 0$.)
 
We denote absolute loss using the notation $\ell (x, y, f) = |f(x) - y|$. Given a set of samples $\sample = \left\{\left(x_1, y_1\right), \dots, \left(x_N, y_N\right)\right\} \subseteq \cX \times [0,1]$, we use the standard notation $L_{\sample}(f) = \frac{1}{N} \sum_{i = 1}^N \left|f\left(x_i\right) - y_i\right|$ to denote the average empirical loss of a function $f: \cX \to [0,1]$ and $L_{\dist}(f) = \E_{(x,y) \sim \dist}[|f(x)-y|]$ to denote the expected loss of $f$.
The absolute loss function can be naturally incorporated into the definition of Rademacher complexity: $\erad(\ell \circ \cF) = \frac{1}{N} \E_{\vec{\sigma} \sim \{-1,1\}^N} \left[\sup_{f \in \cF} \sum_{i = 1}^N \sigma_i \left|f\left(x_i\right) - y_i\right|\right].$ The \emph{worst-case empirical Rademacher complexity} of a class $\cF$ is defined as $\wcrad_N(\ell \circ \cF) = \sup_{\sample : |\sample| = N} \erad(\ell \circ \cF)$.

We now review several standard definitions from learning theory, beginning with that of a learning algorithm.
\begin{definition}[Learning algorithm] A \emph{learning algorithm} $\cA$ takes as input a set $\sample \subseteq \cX \times [0,1]$ of examples and returns a function $\cA_{\sample} : \cX \to [0,1]$.
\end{definition}
As we described earlier in this section, in the integer programming example, we study learning algorithms $\cA$ where $\cA_{\sample} = f_{\hat{\vec{r}}}$ for some CPLEX parameter setting $\hat{\vec{r}} \in \R^d$.

A function class $\cF \subseteq [0,1]^{\cX}$ is \emph{statistically learnable}~\citep{Haussler92:Decision} if there exists some algorithm $\cA$ whose expected loss $L_{\dist}\left(A_{\sample}\right)$ converges to the loss of the best function in $\cF$, $ \inf_{f \in \cF}L_{\dist}(f)$, even for a worst-case distribution $\dist$. We formalize this notion below.
\begin{definition}[Statistical learnability]\label{def:statistical}
	Let $\cF$ be a set of functions mapping $\cX$ to $[0,1]$ and let $\cV_N(\cF) = \inf_{\cA}\sup_{\dist} \E_{\sample \sim \dist^N}\left[L_{\dist}\left(\cA_{\sample}\right) - \inf_{f \in \cF}L_{\dist}(f)\right].$ The function class $\cF$ is \emph{statistically learnable} if $\lim_{N \to \infty} \cV_N(\cF) = 0$.
\end{definition}
In the integer programming example, suppose the class $\cF = \left\{f_{\vec{r}} \mid \vec{r} \in \R^d\right\}$ is statistically learnable. Then there exists a learning algorithm $\cA$ that returns a CPLEX parameter setting $\hat{\vec{r}}$, or equivalently, a function $\cA_{\sample} = f_{\hat{\vec{r}}} \in \cF$, such that the size of the trees CPLEX builds using the parameter setting $\hat{\vec{r}}$ is small in expectation over $\dist$ when compared with the best choice of a parameter setting.

In this work, we study a relaxation of statistical learnability, which we refer to as $\gamma$-statistical learnability. A function class $\cF$ is $\gamma$-statistically-learnable if there exists an algorithm whose expected loss converges to the loss of the best function in $\cF$, plus an additive error term $\gamma$.
\begin{definition}[$\gamma$-statistically learnable]\label{def:gamma} Let $\cF$ be a class of functions mapping $\cX$ to $[0,1]$. The class is \emph{$\gamma$-statistically learnable} if $\lim_{N \to \infty} \cV_N(\cF) \leq \gamma$.
\end{definition}

Based on Theorem~\ref{thm:rad_gen}, it is well-known and easy-to-see that if the worst-case empirical Rademacher complexity of the function class $\cF$ converges to zero as the number of samples grows, then the class $\cF$ is statistically learnable. In other words, if $\lim_{N \to \infty} \wcrad_N(\ell \circ \cF) = 0$, then $\lim_{N \to \infty} \cV_N(\cF) = 0$.

In our integer programming example, suppose the Rademacher complexity of the class $\cF$ is small. Theorem~\ref{thm:rad_gen} guarantees that with high probability over the draw a set of $N$ IPs $\sample = \left\{x_1, \dots, x_N\right\} \sim \dist|_{\cX}^N$, for every choice of a CPLEX parameter vector $\vec{r} \in \R^d$, the size of the tree CPLEX builds when parameterized by $\vec{r}$ on average over the IPs in $\sample$ is close to the size of the tree CPLEX builds in expectation over the draw of an IP $x \sim \dist|_{\cX}$.

If a function class's Rademacher complexity does not converge to zero, then the class is not statistically learnable. We provide an example of one such negative result below.
\begin{theorem}[\citet{Sridharan12:Learning}]\label{thm:rad_lb}
	For any $N \geq 1$ and $\cF \subseteq [0,1]^{\cX}$, $\cV_N(\cF) \geq \wcrad_{2N}(\ell \circ \cF) - \frac{1}{2}\wcrad_N(\ell \circ \cF).$
\end{theorem}
Theorem~\ref{thm:rad_lb} demonstrates that if $\wcrad_N(\ell \circ \cF)$ does not converge to zero, then $\cV_N(\cF)$ will not converge to zero either.

%% file: appendix_generalization.tex
\mainRad*

\begin{proof} Let $\sample = \left\{x_1, \dots, x_N\right\}$ be an arbitrary subset of $\cX$.
	Fix an arbitrary vector $\vec{r} \in \cR$ and index $i \in [N]$. Suppose that $\sigma_i = 1$. Since $f_{x_i}^*(\vec{r}) \leq g_{x_i}^*(\vec{r}) + \norm{f_{x_i}^* - g_{x_i}^*}_{\infty}$, we have that \begin{equation}\sigma_i f_{x_i}^*(\vec{r}) \leq \sigma_i g_{x_i}^*(\vec{r}) + \norm{f_{x_i}^* - g_{x_i}^*}_{\infty}.\label{eq:sigma_positive}\end{equation} Meanwhile, suppose $\sigma_i = -1$. Since $f_{x_i}^*(\vec{r}) \geq g_{x_i}^*(\vec{r}) - \norm{f_{x_i}^* - g_{x_i}^*}_{\infty}$, we have that \begin{equation}\sigma_i f_{x_i}^*(\vec{r}) = -f_{x_i}^*(\vec{r}) \leq -g_{x_i}^*(\vec{r}) + \norm{f_{x_i}^* - g_{x_i}^*}_{\infty} = \sigma_i g_{x_i}^*(\vec{r}) + \norm{f_{x_i}^* - g_{x_i}^*}_{\infty}.\label{eq:sigma_negative}\end{equation} Combining Equations \eqref{eq:sigma_positive} and \eqref{eq:sigma_negative}, we have that \begin{equation}\sup_{\vec{r} \in \cR} \sum_{i = 1}^N \sigma_i g_{\vec{r}}\left(x_i\right) \geq \sum_{i = 1}^N\sigma_i g_{x_i}^*(\vec{r}) \geq\sum_{i = 1}^N\sigma_i f_{x_i}^*(\vec{r}) -  \norm{f_{x_i}^* - g_{x_i}^*}_{\infty}.\label{eq:sup_consequence}\end{equation}
	By definition of the supremum, Equation~\eqref{eq:sup_consequence} implies that for every $\vec{\sigma} \in \{-1,1\}^N$, \[\sup_{\vec{r} \in \cR} \sum_{i = 1}^N \sigma_i g_{\vec{r}}\left(x_i\right) \geq \sup_{\vec{r} \in \cR} \sum_{i = 1}^N \sigma_i f_{\vec{r}}\left(x_i\right)  - \sum_{i = 1}^N\norm{f_{x_i}^* - g_{x_i}^*}_{\infty}.\] Therefore \[\E_{\vec{\sigma} \sim \{-1,1\}^N}\left[\sup_{\vec{r} \in \cR} \sum_{i = 1}^N \sigma_i g_{\vec{r}}\left(x_i\right)\right] \geq \E_{\vec{\sigma} \sim \{-1,1\}^N}\left[\sup_{\vec{r} \in \cR} \sum_{i = 1}^N \sigma_i f_{\vec{r}}\left(x_i\right)\right]  - \sum_{i = 1}^N\norm{f_{x_i}^* - g_{x_i}^*}_{\infty},\] so the lemma statement holds.
\end{proof}

%% file: appendix_SRM.tex
\corSRM*

\begin{proof}
	We will prove that with probability at least $1-\delta$ over the draw of the training set $\sample = \left\{\left(x_1, y_1\right), \dots, \left(x_N, y_N\right)\right\} \sim \dist^N$, for all parameter vectors $\vec{r} \in \cR$ and all $j \in \N$,
	\[\left|L_{\sample}\left(f_{\vec{r}}\right) - L_{\dist}\left(f_{\vec{r}}\right)\right| \leq \frac{2}{N} \sum_{i = 1}^N  \norm{f_{x_i}^* - g_{j, x_i}^*}_{\infty} + 2\erad\left(\ell \circ \cG_j\right) + 3\sqrt{\frac{1}{2N}\ln \frac{(\pi j)^2}{3\delta}}.\]

	For each integer $j \geq 1$, let $\delta_j = \frac{6\delta}{(\pi j)^2}$. From Theorems~\ref{thm:rad_gen} and \ref{thm:rad}, we know that with probability at least $1-\delta_j$ over the draw of the training set $\sample = \left\{\left(x_1, y_1\right), \dots, \left(x_N, y_N\right)\right\} \sim \dist^N$, for all parameter vectors $\vec{r} \in \cR$, \begin{align*}\left|L_{\sample}\left(f_{\vec{r}}\right) - L_{\dist}\left(f_{\vec{r}}\right)\right| &\leq \frac{2}{N} \sum_{i = 1}^N  \norm{f_{x_i}^* - g_{j, x_i}^*}_{\infty} + 2\erad\left(\ell \circ \cG_j\right) + 3\sqrt{\frac{1}{2N}\ln \frac{2}{\delta_j}}\\
	&= \frac{2}{N} \sum_{i = 1}^N  \norm{f_{x_i}^* - g_{j, x_i}^*}_{\infty} + 2\erad\left(\ell \circ \cG_j\right) + 3\sqrt{\frac{1}{2N}\ln \frac{(\pi j)^2}{3\delta}}.\end{align*} Since $\sum_{i = 1}^{\infty} \delta_j = \delta$, the corollary follows from a union bound over all $j \geq 1$.
\end{proof}

\begin{cor}\label{cor:SRM_exp}
	Let $\cF = \left\{f_{\vec{r}} \mid \vec{r} \in \cR\right\} \subseteq [0,1]^{\cX}$ be a set of functions mapping $\cX$ to $[0,1]$. Let $\cG_1, \cG_2, \cG_3, \dots$ be a countable sequence of function classes, where for each $j \in \N$, $\cG_j = \left\{g_{j,\vec{r}} \mid \vec{r} \in \cR\right\} \subseteq [0,1]^{\cX}$ is a set of functions mapping $\cX$ to $[0,1]$, parameterized by vectors $\vec{r} \in \cR$.
	With probability at least $1-\delta$ over the draw of the training set $\sample = \left\{\left(x_1, y_1\right), \dots, \left(x_N, y_N\right)\right\} \sim \dist^N$, for all parameter vectors $\vec{r} \in \cR$ and all $j \in \N$,
	\[\left|L_{\sample}\left(f_{\vec{r}}\right) - L_{\dist}\left(f_{\vec{r}}\right)\right| \leq 2\erad\left(\ell \circ \cG_j\right) + 2\E_{x \sim \dist|_{\cX}} \left[\norm{f_{x}^* - g_{j,x}^*}_{\infty}\right] + \sqrt{\frac{2}{N}\ln \frac{2(\pi j)^2}{3\delta}}.\]
\end{cor}

\begin{proof}
 From Theorem~\ref{thm:rad}, we know that for every integer $j \geq 1$, \[\E_{\sample' \sim \dist^N}\left[\widehat{\mathcal{R}}_{\sample'}(\ell \circ \cF)\right] \leq \E_{\sample' \sim \dist^N}\left[\widehat{\mathcal{R}}_{\sample'}\left(\ell \circ \cG_j\right)\right] + \E_{x \sim \dist|_{\cX}} \left[\norm{f_{x}^* - g_{j,x}^*}_{\infty}\right].\]
	For each integer $j \geq 1$, let $\delta_j = \frac{6\delta}{(\pi j)^2}$.
	 From Theorem~\ref{thm:exp_rad} and Hoeffding bound, we know that with probability at least $1-\delta_j$ over the draw of the training set $\sample = \left\{\left(x_1, y_1\right), \dots, \left(x_N, y_N\right)\right\} \sim \dist^N$, for all parameter vectors $\vec{r} \in \cR$, \begin{align*}\left|L_{\sample}\left(f_{\vec{r}}\right) - L_{\dist}\left(f_{\vec{r}}\right)\right| &\leq 2\erad\left(\ell \circ \cG_j\right) + 2\E_{x \sim \dist|_{\cX}} \left[\norm{f_{x}^* - g_{j,x}^*}_{\infty}\right] + \sqrt{\frac{2}{N}\ln \frac{4}{\delta_j}}\\
	&= 2\erad\left(\ell \circ \cG_j\right) + 2\E_{x \sim \dist|_{\cX}} \left[\norm{f_{x}^* - g_{j,x}^*}_{\infty}\right] + \sqrt{\frac{2}{N}\ln \frac{2(\pi j)^2}{3\delta}}.\end{align*} Since $\sum_{i = 1}^{\infty} \delta_j = \delta$, the corollary follows from a union bound over all $j \geq 1$.
\end{proof}

\begin{theorem}[e.g., \citet{Mohri12:Foundations}]\label{thm:exp_rad}
	Let $\cF \subseteq [0,1]^{\cX}$ be a set of functions mapping a domain $\cX$ to $[0,1]$. With probability at least $1-\delta$ over the draw of $N$ samples $\sample = \left\{\left(x_1, y_1\right), \dots, \left(x_N, y_N\right)\right\} \sim \dist^N$, the following holds for all $f \in \cF$:
\[\left|L_{\sample}\left(f_{\vec{r}}\right) - L_{\dist}\left(f_{\vec{r}}\right)\right| \leq 2\E_{\sample' \sim \dist^N}\left[\widehat{\mathcal{R}}_{\sample'}(\ell \circ \cF)\right] + \sqrt{\frac{1}{2N} \ln \frac{2}{\delta}}.\]
\end{theorem}

In the following lemma, we show that for any function classes $\cF = \left\{f_{\vec{r}} \mid \vec{r} \in \cR\right\} \subseteq [0,1]^{\cX}$ and $\cG = \left\{g_{\vec{r}} \mid \vec{r} \in \cR\right\} \subseteq [0,1]^{\cX}$, the value $\E_{(x,y) \sim \dist}\left[\norm{f_x^* - g_x^*}_{\infty}\right]$, which appears in the generalization guarantee in from Corollary~\ref{cor:SRM_exp}, can be estimated from samples.

\begin{lemma}\label{lem:estimate}
	Let $\cF = \left\{f_{\vec{r}} \mid \vec{r} \in \cR\right\} \subseteq [0,1]^{\cX}$ and $\cG = \left\{g_{\vec{r}} \mid \vec{r} \in \cR\right\} \subseteq [0,1]^{\cX}$ be two sets of functions mapping a domain $\cX$ to $[0,1]$. With probability $1- \delta$ over the draw of $N$ samples $\left(x_1, y_1\right), \dots, \left(x_N, y_N\right) \sim \dist$, \begin{equation}\E_{(x,y) \sim \dist}\left[\norm{f_x^* - g_x^*}_{\infty}\right] \leq \frac{1}{N} \sum_{i = 1}^N \norm{f_{x_i}^* - g_{x_i}^*}_{\infty} + \sqrt{\frac{1}{2N} \ln \frac{1}{\delta}}.\label{eq:Hoeffding}\end{equation}
\end{lemma}

\begin{proof}
	Let $h : \cX \times [0,1] \to [0,1]$ be defined such that $h(x,y) = \norm{f_x^* - g_x^*}_{\infty}$. From Hoeffding's inequality, we know that with probability $1- \delta$ over the draw of $N$ samples $\left(x_1, y_1\right), \dots, \left(x_N, y_N\right) \sim \dist$, \[\E_{(x,y) \sim \dist}\left[h(x,y)\right] \leq \frac{1}{N} \sum_{i = 1}^N h\left(x_i, y_i\right) + \sqrt{\frac{1}{2N} \ln \frac{1}{\delta}},\] which implies that Equation~\eqref{eq:Hoeffding} holds.
\end{proof}

%% file: appendix_IP.tex
\paragraph{Selecting a tree size upper bound.}
As we described earlier in this section, we assume there is an upper bound $\kappa$ on the size of the largest tree we allow branch-and-bound to build before we terminate, as in prior research~\citep{Hutter09:ParamILS, Kleinberg17:Efficiency, Balcan18:Learning, Kleinberg19:Procrastinating}. Given a parameter setting $r \in [0,1]$ and an integer program $x \in \cX$, we define $f_r(x)$ to be the size of the tree CPLEX builds, capped at $\kappa$, divided by $\kappa$ (this way, $f_r(x)$ is normalized, contained in the interval $[0,1]$).

We use a data-dependent approach to select $\kappa$. For any parameter $r \in [0,1]$ and integer program $x \in \cX$, let $h_r(x)$ be the size of the tree CPLEX builds (unnormalized). We draw $N = 6000$ integer programs $x_1, \dots, x_N$ from the underlying distribution $\dist$ and set $\kappa = \max_{r \in [0,1], i \in [N]} h_r\left(x_i\right)$.
Classic results from learning theory guarantee that with high probability, for at most 8\% of the integer programs sampled from $\dist$, CPLEX will build a tree of size larger than $\kappa$ when parameterized by some $r \in [0,1]$. Specifically, since the VC dimension of threshold functions is 1, we have that
with probability at least 0.99 over the draw of the $N$ samples, $\Pr_{x \sim \dist}\left[\max_{r \in [0,1]}f_r(x) > \kappa\right] < 0.08$.

For the ``arbitrary'' distribution, when $\score_1 = \score_L$ and $\score_2 = \score_S$, $\kappa = 6341$, and when $\score_1 = \score_P$ and $\score_2 = \score_A$, $\kappa = 2931$. For the ``regions'' distribution, when $\score_1 = \score_L$ and $\score_2 = \score_S$, $\kappa = 7314$.

\paragraph{Dynamic programming.} For any $k \in \N$, let $\cG_k$ be the set of piecewise-constant functions with $k$ pieces mapping an interval $\cR \subseteq \R$ to $\R$.
In this section, we provide a dynamic programming algorithm which takes as input a piecewise-constant dual function $f_x^*: \cR \to \R$ and a value $k \in \N$ and returns the value $\min_{g \in \cG_k} \norm {f_x^* - g}_{\infty}$. Since $f_x^*$ is piecewise-constant, the domain $\cR$ can be partitioned into intervals $\left[a_1, a_2\right), \left[a_2, a_3\right) \dots, \left[a_t, a_{t+1}\right)$ such that for any interval $\left[a_i, a_{i+1}\right)$, there exists a value $c_i \in \R$ such that $f_x^*(r) = c_i$ for all $r \in \left[a_i, a_{i+1}\right)$. 

We now provide an overview of the algorithm. See Algorithm~\ref{alg:DP} for the pseudo-code.
\begin{algorithm}
	\caption{Piecewise-constant function fitting via dynamic programming}\label{alg:DP}
	\begin{algorithmic}
		\State {\bfseries Input:} Partition $\left[a_1, a_2\right), \dots, \left[a_t, a_{t+1}\right)$ of $\cR$, values $c_1, \dots, c_t$, and desired number of pieces $k \in \N$.
		\For {$i \in [t]$}\label{step:upper_lower_for_begin}
		\State Set $u_{i,i} = c_i$ and $\ell_{i,i} = c_i$.
		\For {$i' \in \{i+1, \dots, t\}$}
		\If {$c_{i'} < \ell_{i,i'-1}$}
		\State Set $\ell_{i,i'} = c_{i'}$ and $u_{i,i'} = u_{i,i'-1}$.
		\ElsIf {$c_{i'} > u_{i,i'-1}$}
		\State Set $\ell_{i,i'} = \ell_{i,i'-1}$ and $u_{i,i'} = c_{i'}$.
		\Else{}
		\State Set $\ell_{i,i'} = \ell_{i,i'-1}$ and $u_{i,i'} = u_{i,i'-1}$.
		\EndIf
		\EndFor
		\EndFor\label{step:upper_lower_for_end}
		\For {$i \in [t]$}\label{step:DP_begin}
		\State Set $C(i, 1) = \frac{u_{1,i} - \ell_{1,i}}{2}$
		\EndFor
		\For {$j \in \{2, \dots, k\}$}
		\For {$i \in [t]$}
		\State Set $C(i,j) = \min \left\{C(i,1), \min_{i' \in [i-1]} \left\{C(i', j-1) + \frac{u_{i'+1,i} - \ell_{i'+1,i}}{2}\right\}\right\}$
		\EndFor
		\EndFor\label{step:DP_end}
		\State {\bfseries Output:} $C(t,k)$.
	\end{algorithmic}
\end{algorithm}

The algorithm takes as input the partition $\left[a_1, a_2\right), \dots, \left[a_t, a_{t+1}\right)$ of the parameter space $\cR$ and values $c_1, \dots, c_t$ such that for any interval $\left[a_i, a_{i+1}\right)$, $f_x^*(r) = c_i$ for all $r \in \left[a_i, a_{i+1}\right)$. The algorithm begins by calculating upper and lower bounds on the value of the function $f_x^*$ across various subsets of its domain. In particular, for each $i,i' \in [t]$ such that $i \leq i'$, the algorithm calculates the lower bound $\ell_{i,i'} = \min\left\{c_i, c_{i+1}, \dots, c_{i'}\right\}$ and the upper bound $u_{i,i'} = \max\left\{c_i, c_{i+1}, \dots, c_{i'}\right\}$. Algorithm~\ref{alg:DP} performs these calculations in $O(t^2)$ time.

Next, for each $i \in [t]$ and $j \in [k]$, the algorithm calculates a value $C(i,j)$ which equals the smallest $\ell_{\infty}$ norm between any piecewise constant function with $j$ pieces and the function $f_x^*$ when restricted to the interval $\left[a_1, a_{i+1}\right)$. Since $\cR = \left[a_1, a_{t+1}\right)$, we have that $C(t,k)$---the value our algorithm returns---equals $\min_{g \in \cG_k} \norm {f_x^* - g}_{\infty}$, as claimed. For all $i \in [t]$, $C(i,1) = \frac{u_{1,i} - \ell_{1,i}}{2}$ and for all $j \geq 2$, \[C(i,j) = \min \left\{C(i,1), \min_{i' \in [i-1]} \left\{C(i', j-1) + \frac{u_{i'+1,i} - \ell_{i'+1,i}}{2}\right\}\right\}.\]  Algorithm~\ref{alg:DP} performs these calculations in $O(kt^2)$ time.

\paragraph{Additional lemmas.}

\begin{lemma}\label{lem:PWC}
	Let $\cG = \left\{g_r \mid r \in \R\right\} \subseteq [0,1]^{\cX}$ be a set of functions mapping a set $\cX$ to $[0,1]$ parameterized by a single real value $r \in \R$. Suppose that every function $g_x^* \in \cG^* \subseteq [0,1]^{\R}$ is piecewise-constant with at most $j$ pieces. Then for any set $\sample = \left\{x_1, \dots, x_N\right\} \subseteq \cX$, \[\erad\left(\cG\right) = \frac{1}{N} \E_{\vec{\sigma} \sim \{-1,1\}^N} \left[\sup_{r \in \R} \sum_{i = 1}^N \sigma_i g_r\left(x_i\right)\right] \leq \sqrt{\frac{2 \ln(N(j-1) + 1)}{N}}.\]
\end{lemma}

\begin{proof}
	We will use Massart's lemma (Lemma~\ref{lem:massart}) to prove this lemma. Let $A \subseteq [0,1]^N$ be the following set of vectors:
	\[A = \left\{\begin{pmatrix}
	g_{r}\left(x_1\right)\\
	\vdots\\
	g_{r}\left(x_N\right)
	\end{pmatrix} : r \in \R\right\}.\] By definition of the dual class, \[A = \left\{\begin{pmatrix}
	g_{x_1}^*\left(r\right)\\
	\vdots\\
	g_{x_N}^*\left(r\right)
	\end{pmatrix} : r \in \R\right\}.\] Since each function $g_{x_i}^*$ is piecewise-constant with at most $j$ pieces, $|A| \leq N(j-1) + 1$. The lemma statement therefore follows from Massart's lemma.
\end{proof}

\begin{lemma}[\citet{Massart00:Some}]\label{lem:massart}
	Let $A \subseteq [0,1]^N$ be a finite set of vectors. Then \[\frac{1}{N} \E_{\vec{\sigma} \sim \{-1,1\}^N} \left[\sup_{\vec{a} \in A} \sum_{i = 1}^N \sigma_i a_i\right] \leq \sqrt{\frac{2 \ln |A|}{N}}.\]
\end{lemma}

\subsubsection{Additional experiments}\label{app:experiments}

In our experiments from Section~\ref{sec:IP}, we approximated the dual functions $f_x^*$ with piecewise constant functions that have a small number of pieces --- say, $j$ pieces. We used SRM to find the value for $j$ which leads to the strongest bounds, as in Equation~\eqref{eq:estimate_gen}. In this section, we compare against another baseline where we do not use SRM, but simply set $j$ to be the maximum number of pieces we observe over our training set. Of course, this bound is much tighter than the worst-case bound by \citet{Balcan18:Learning}, the baseline in  Figures~\ref{fig:regions}-\ref{fig:arbitrary_product}. However, we still observe that for a target generalization error, the number of samples required according to our bound is up to 4.5 times smaller than the number of samples required by this baseline.

For each of the three experimental setups from  Figures~\ref{fig:regions}-\ref{fig:arbitrary_product}, we draw $M = 6000$ IPs
$x_1, \dots, x_M$ from the distribution $\dist$. We compute the piecewise-constant dual functions $f_{x_1}^*, \dots, f_{x_M}^*$ and find the maximum number of pieces $j^*$ across these $M$ functions. We summarize our findings below:
\begin{itemize}
	\item When using the CATS ``arbitrary'' generator with $\score_1 = \score_L$ and $\score_2 = \score_S$, the maximum number of pieces is $j^* = 2214$.
		\item When using the CATS ``arbitrary'' generator with $\score_1 = \score_P$ and $\score_2 = \score_A$, the maximum number of pieces is $j^* = 296$.
				\item When using the CATS ``regions'' generator with $\score_1 = \score_L$ and $\score_2 = \score_S$, the maximum number of pieces is $j^* = 2224$.
\end{itemize}

Since there is a piecewise-constant function $g_{j^*,x_i}^*$ with at most $j^*$ pieces that exactly equals each dual function $f_{x_i}^*$, a Hoeffding bound guarantees that with probability 0.995, $\E_{x \sim \dist}\left[\norm{f_x^* - g_{j^*,x}^*}_{\infty}\right] \leq 0.023$. Therefore, from Theorem~\ref{thm:rad_gen}, Theorem~\ref{thm:rad}, Remark~\ref{remark:exp}, and Theorem~\ref{lem:IP_PWC}, we know that with probability $0.99$ over the draw of $N$ samples $\sample \sim \dist^N$, for all $r \in [0,1]$, \begin{equation}\left|\frac{1}{N} \sum_{x \in \sample} f_{r}(x) - \E_{x \sim \dist} \left[f_{r}(x)\right]\right| \leq 2\left(0.023 + \sqrt{\frac{2 \ln(N(j^*-1) + 1)}{N}}\right) + 3\sqrt{\frac{1}{2N}\ln \frac{2}{0.005}}.\label{eq:baseline}\end{equation}
\begin{figure}[t]
	\centering
	\subfigure[Results on the CATS ``regions'' generator with $\score_1 = \score_L$ and $\score_2 = \score_S$.\label{fig:regions_app}]{\includegraphics[width=0.47\textwidth]{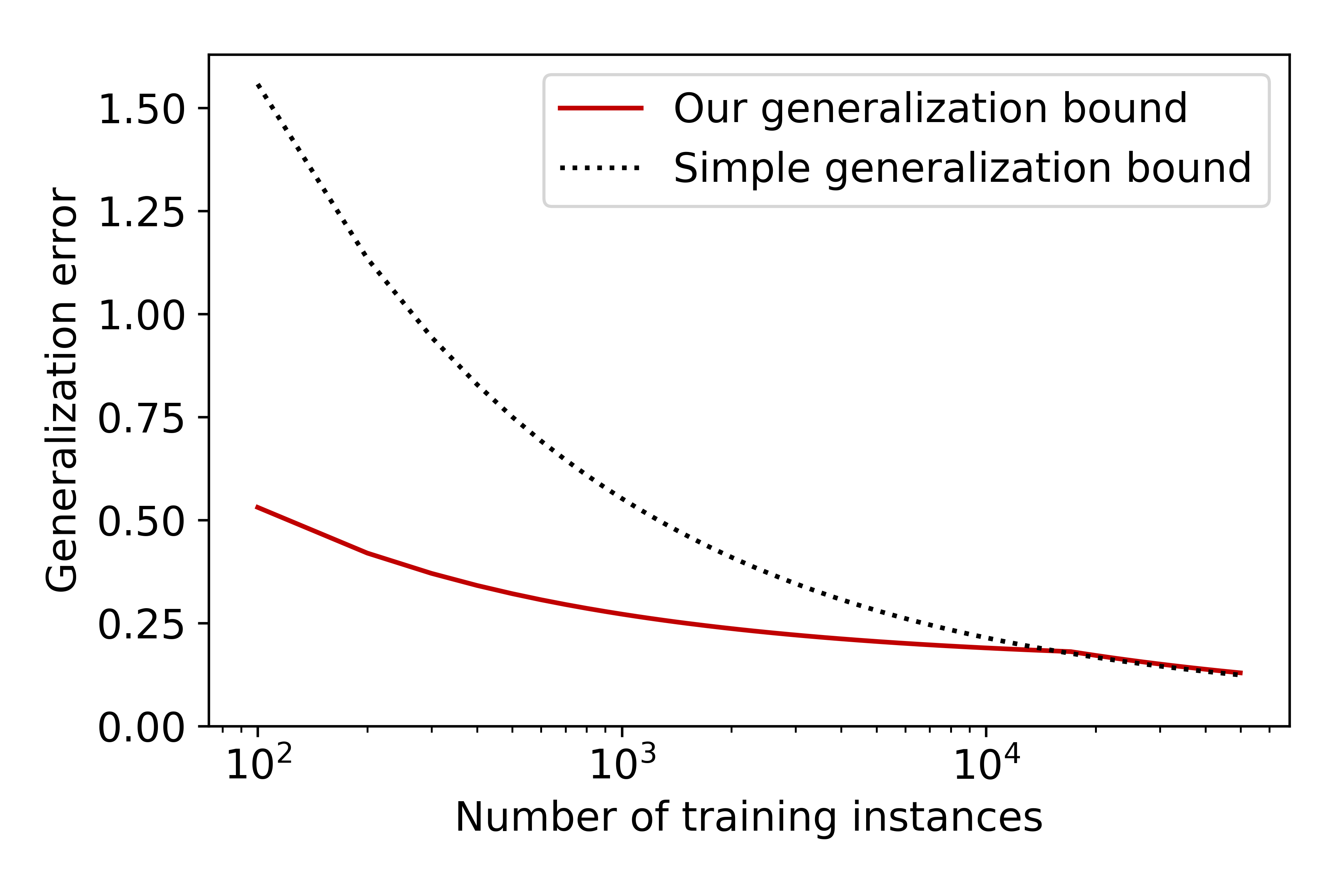}}\qquad
	\subfigure[Results on the CATS ``arbitrary'' generator with $\score_1 = \score_L$ and $\score_2 = \score_S$.\label{fig:arbitrary_app}]{\includegraphics[width=0.47\textwidth]{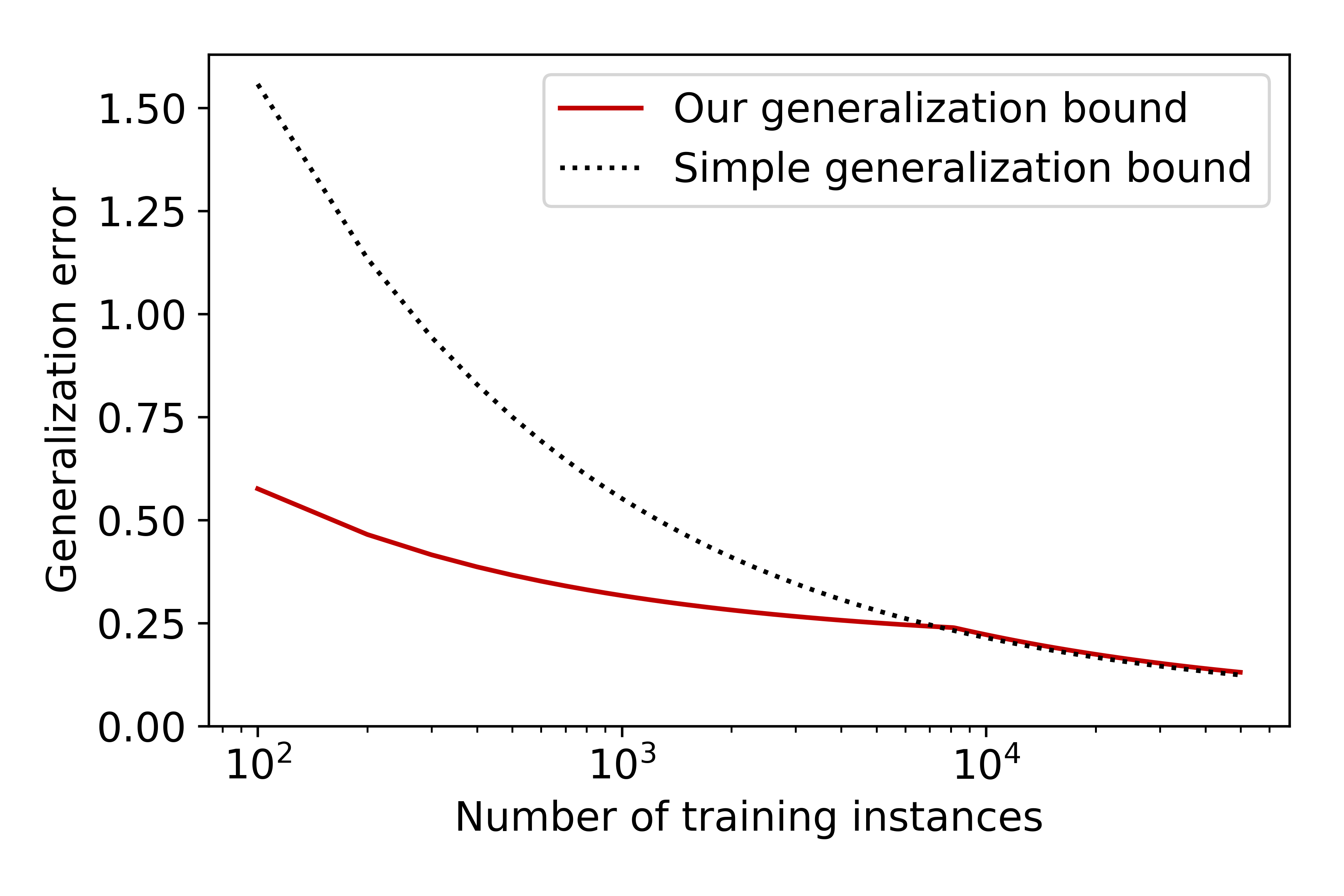}}
	\subfigure[{Results on the CATS ``arbitrary'' generator with $\score_1 = \score_P$ and $\score_2 = \score_A$.\label{fig:arbitrary_product_app}}]{\includegraphics[width=0.47\textwidth]{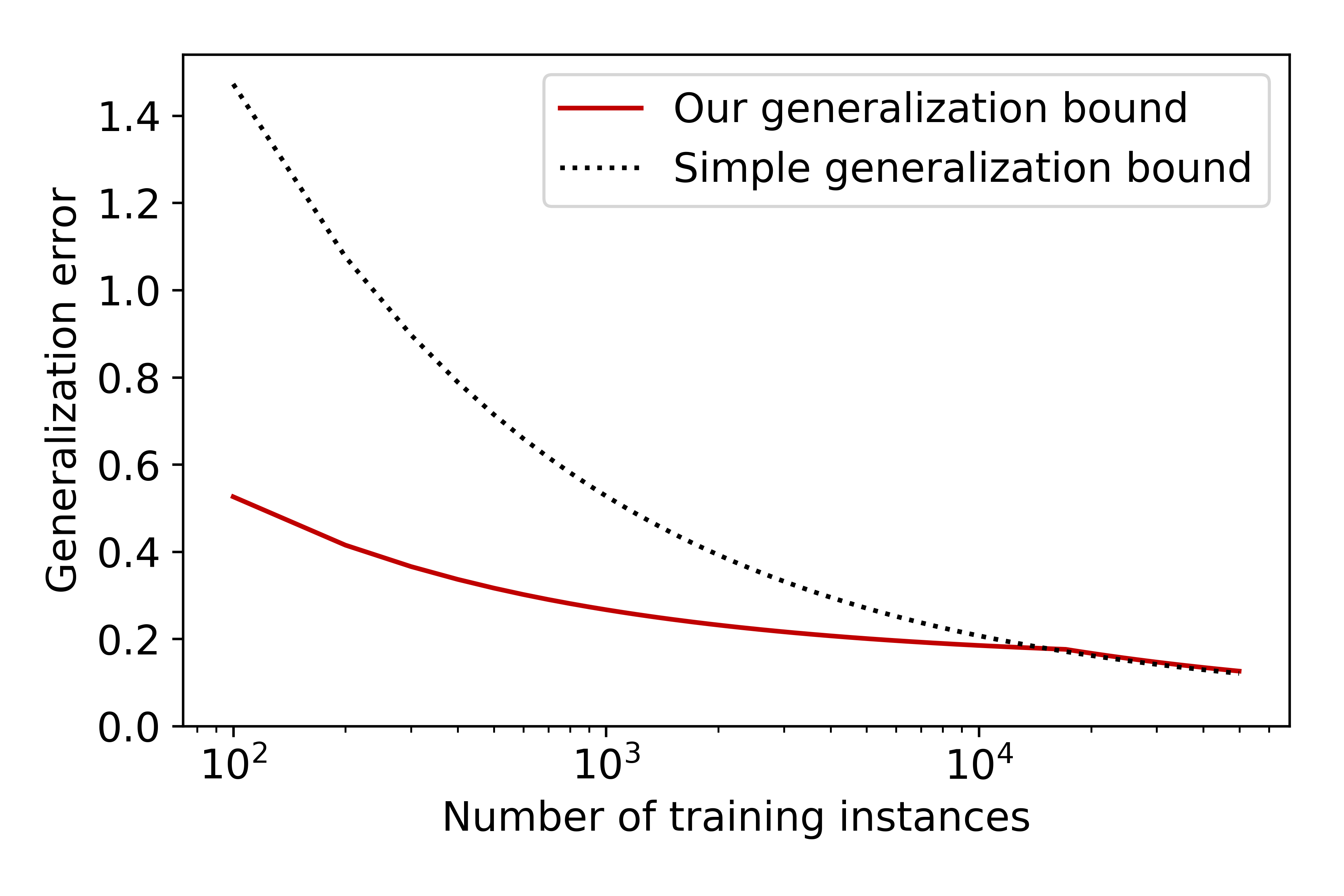}}
	\caption{Experiments where we compare our generalization bound to a simple baseline described in Section~\ref{app:experiments}. The red solid line is our generalization bound: the minimum of Equations~\eqref{eq:WC_gen} and \eqref{eq:estimate_gen} as a function of the number of training examples $N$. The black dotted line is the simple baseline from Equation~\eqref{eq:baseline}.}
	\label{fig:experiments_app}
\end{figure}
This is the black dotted line in Figure~\ref{fig:experiments_app}. The red solid line is our generalization bound, as we described in Section~\ref{sec:IP}: the minimum of Equations~\eqref{eq:WC_gen} and \eqref{eq:estimate_gen}.

In Figure~\ref{fig:experiments_app}, we see that our bound significantly beats this simple baseline up until the point there are approximately 10,000 training instances, at which point they are approximately equal. These experiments demonstrate that for a target generalization error, the number of samples required according to our bound is up to 4.5 times smaller than the number of samples required by this baseline. In Figure~\ref{fig:regions_app}, to get a generalization error of 0.25, 1500 samples are sufficient our approach and 6700 samples are sufficient using the baseline, so we see a 4.6x improvement. In Figure~\ref{fig:arbitrary_app}, to get a generalization error of 0.3, 1400 samples are sufficient our approach and 4300 samples are sufficient using the baseline, so we see a 3.07x improvement. Finally, in Figure~\ref{fig:arbitrary_product_app}, to get a generalization error of 0.25, 1400 samples are sufficient our approach and 6100 samples are sufficient using the baseline, so we see a 4.35x improvement.

%% file: appendix_not_learnable.tex
\begin{theorem}[H\"{o}lder's inequality]\label{thm:holder}
	Let $p_0$ and $p_1$ be two values in $[1, \infty]$ such that $\frac{1}{p_0} + \frac{1}{p_1} = 1$. Then for all functions $u$ and $w$, $\norm{uw}_1 \leq \norm{u}_{p_0} \norm{w}_{p_1}$.
\end{theorem}

\begin{theorem}[Interpolation]\label{thm:interpolation}
	Let $p$ and $q$ be two values in $(0, \infty]$ and let $\theta$ be a value in $(0,1)$. Let $p_{\theta}$ be defined such that $\frac{1}{p_{\theta}} = \frac{\theta}{p_1} + \frac{1 - \theta}{p_0}$. Then for all functions $u$, $\norm{f}_{p_{\theta}} \leq \norm{f}_{p_1}^\theta \norm{f}_{p_0}^{1 - \theta}$.
\end{theorem}

\begin{lemma}\label{lem:approximates}
	For any $\gamma \in \left(0, \frac{1}{4}\right)$ and $p \in [1, \infty)$, let $\cF$ and $\cG$ be the function classes defined in Theorem~\ref{thm:not_learnable}.
	The dual class $\cG^*$ $(\gamma, p)$-approximates the dual class $\cF^*$.
\end{lemma}
\begin{proof} For ease of notation, let $t = \gamma^p$, $a = \frac{1}{2\gamma^p}$, $\cR = (0,t]$, and $\cX = \left[\frac{1}{2\gamma^p}, \infty\right)$. Throughout this proof, we will use the following inequality: \begin{equation}\norm{f_x^* - g_x^*}_2 = \sqrt{\int_0^t \left(f_x^*(r) - g_x^*(r)\right)^2 \, dr}
	= \sqrt{\int_0^t \left(\frac{1}{2}\cos(rx)\right)^2 \, dr} = \frac{1}{4} \sqrt{2 t + \frac{\sin(2 t x)}{x}} \leq \frac{1}{4} \sqrt{2 t + \frac{1}{x}}\label{eq:approximates}.\end{equation}
	
	First, suppose $p = 2$. Since $t = \gamma^2$ and $\frac{1}{x} \leq 2\gamma^2$, Equation~\eqref{eq:approximates} implies that $\norm{f_x^* - g_x^*}_2 \leq \frac{1}{4} \sqrt{4\gamma^2} < \gamma$.
	
	Next, suppose $p < 2$. We know that \begin{equation}\norm{\left(f_x^* - g_x^*\right)^p}_1 = \int_0^t \left|\left(f_x^*(r) - g_x^*(r)\right)^p\right| \, dr = \int_0^t \left|f_x^*(r) - g_x^*(r)\right|^p \, dr = \norm{f_x^* - g_x^*}_p^p.\label{eq:l1}\end{equation} From Equation~\eqref{eq:l1} and H\"older's inequality (Theorem~\ref{thm:holder}) with $u = \left(f_x^* - g_x^*\right)^p$, $w$ the constant function $w: r \mapsto 1$, $p_0 = \frac{2}{p}$, and $p_1 = \frac{2}{2-p}$, we have that
	\begin{align*}
	\norm{f_x^* - g_x^*}_p^p &= \norm{\left(f_x^* - g_x^*\right)^p}_1\\
	&\leq \norm{w}_{\frac{2}{2-p}}\norm{\left(f_x^* - g_x^*\right)^p}_{\frac{2}{p}}\\
	&= \left(\int_0^t \, dr\right)^{\frac{2-p}{2}}\norm{\left(f_x^* - g_x^*\right)^p}_{\frac{2}{p}}\\
	&= t^{\frac{2-p}{2}}\norm{\left(f_x^* - g_x^*\right)^p}_{\frac{2}{p}}\\
	&= t^{\frac{2-p}{2}}\left(\int_0^t\left(f_x^*(r) - g_x^*(r)\right)^2 \, dr\right)^{\frac{p}{2}}\\
	&= t^{\frac{2-p}{2}}\norm{f_x^* - g_x^*}_2^p.
	\end{align*} Therefore,
	\begin{align*}\norm{f_x^* - g_x^*}_p &\leq t^{\frac{1}{p} - \frac{1}{2}} \norm{f_x^* - g_x^*}_2\\
	&\leq \frac{t^{\frac{1}{p} - \frac{1}{2}}}{4} \sqrt{2 t + \frac{1}{x}} &(\text{Equation~\eqref{eq:approximates}})\\
	&= \frac{t^{\frac{1}{p}}}{4} \sqrt{2 + \frac{1}{xt}}\\
	&= \frac{\gamma}{4} \sqrt{2 + \frac{1}{x\gamma^p}} &\left(t = \gamma^p\right)\\
	&< \gamma, &\left(x \geq \frac{1}{2\gamma^p}\right)\end{align*} 
	
	Finally, suppose $p > 2$. Let $\theta =1 - \frac{2}{p}$, $p_0 = 2$, and $p_1 = \infty$. By Theorem~\ref{thm:interpolation}, \begin{align*}\norm{f_x^* - g_x^*}_p &\leq \norm{f_x^* - g_x^*}_2^{1 - \theta}\\
	&= \norm{f_x^* - g_x^*}_2^{\frac{2}{p}}\\
	&\leq \sqrt[p]{\frac{t}{8} + \frac{1}{16x}}&(\text{Equation~\eqref{eq:approximates}})\\
	&= \sqrt[p]{\frac{\gamma^p}{8} + \frac{1}{16x}}&\left(t = \gamma^p\right)\\
	&\leq \sqrt[p]{\frac{\gamma^p}{4}}&\left(x \geq \frac{1}{2\gamma^p}\right)\\
	&< \gamma.\end{align*} Therefore, for all $p \in [1, \infty)$ and all $x \in \cX$, $\norm{f_x^* - g_x^*}_p \leq \gamma$, so the dual class $\cG^*$ $(\gamma, p)$-approximates the dual class $\cF^*$.
\end{proof}

\begin{lemma}\label{lem:not_learnable}
	For any $\gamma \in \left(0, \frac{1}{4}\right)$ and $p \in [1, \infty)$, let $\cF = \left\{ f_r \mid r \in \left(0, \gamma^p\right]\right\}$ be a class of functions with domain $\left[\frac{1}{2\gamma^p}, \infty\right)$ such that for all $r \in \left(0, \gamma^p\right]$ and $x \in \left[\frac{1}{2\gamma^p}, \infty\right)$, $f_r(x) = \frac{1}{2}(1 + \cos(rx))$. 
	For every $N \geq 1$, $\wcrad_N(\ell \circ \cF) = \frac{1}{2}$.
\end{lemma}
\begin{proof}
	This proof is similar to the proof that the VC-dimension of the function class \[\left\{x \mapsto \sign(\sin(rx)) \mid r \in \R\right\} \subseteq \{-1,1\}^{\R}\] is infinite (see, for example, Lemma 7.2 in the textbook by \citet{Anthony09:Neural}).
	To prove this lemma, we will show that for every $c \in (0, 1/2)$, $\wcrad_N(\ell \circ \cF) \geq c$ (Claim~\ref{claim:rad_lower}). We also show that $\wcrad_N(\ell \circ \cF) \leq \frac{1}{2}$ (Claim~\ref{claim:rad_upper}). Therefore, the lemma statement follows.
	
	\begin{claim}\label{claim:rad_lower}
		For every $c \in (0, 1/2)$, $\wcrad_N(\ell \circ \cF) \geq c$.
	\end{claim}
	\begin{proof}[Proof of Claim~\ref{claim:rad_lower}] Let $N$ be an arbitrary positive integer. We begin by defining several variables that we will use throughout this proof. Let $\cR = \left(0, \gamma^p\right]$ and let $\alpha$ be any positive power of $\frac{1}{2}$ smaller than $\min\left\{\frac{1}{2\pi + 1}, \frac{\arccos(2c)}{\pi + \arccos(2c)}\right\}$. Since $2c \in (0, 1)$, $\frac{\arccos(2c)}{\pi + \arccos(2c)}$ is well-defined. Also, since $\alpha \leq \frac{\arccos(2c)}{\pi + \arccos(2c)}$, we have that $\frac{\pi\alpha}{1 - \alpha} \leq \arccos(2c) < \frac{\pi}{2}$. Finally, since the function $\cos$ is decreasing on the interval $[0, \pi/2]$, we have that $\frac{1}{2}\cos \frac{\pi \alpha}{1 - \alpha} \geq c$. Let $x_i = \frac{\alpha^{-i}}{2 \gamma^p}$ and $y_i = 0$ for $i \in [N]$. Since $\alpha < 1$, we have that $x_i \geq \frac{1}{2 \gamma^p}$, so each $x_i$ is an element of the domain $\left[\frac{1}{2\gamma^p}, \infty\right)$ of the functions in $\cF$.
		
		We will show that for every assignment of the variables $\sigma_1, \dots, \sigma_N \in \{-1,1\}$, there exists a parameter $r_0 \in \left(0, \gamma^p\right]$ such that \[\frac{1}{N} \sup_{r \in \left(0, \gamma^p\right]}\sum_{i = 1}^N \sigma_i f_r\left(x_i\right) \geq \frac{1}{N} \sum_{i = 1}^N \sigma_i f_{r_0}\left(x_i\right)= \frac{1}{2N} \sum_{i = 1}^N \sigma_i \left(1 + \cos\left(r_0x_i\right)\right)\geq c + \frac{1}{2} \sum_{i = 1}^N \sigma_i.\] This means that when $\sample = \left\{\left(x_1, y_1\right), \dots, \left(x_N, y_N\right)\right\}$, \begin{align*}\wcrad_N(\ell \circ \cF) &\geq \erad(\ell \circ \cF)\\
		&= \frac{1}{N} \E_{\vec{\sigma}}\left[\sup_{r \in \cR}\sum_{i = 1}^N \sigma_i \left|f_r\left(x_i\right) - y_i\right|\right]\\
		&= \frac{1}{N} \E_{\vec{\sigma}}\left[\sup_{r \in \cR}\sum_{i = 1}^N \sigma_i \left|f_r\left(x_i\right)\right|\right] &\left(y_i = 0\right)\\
		&= \frac{1}{N} \E_{\vec{\sigma}}\left[\sup_{r \in \cR}\sum_{i = 1}^N \sigma_i f_r\left(x_i\right)\right] &\left(f_r\left(x_i\right) \geq 0\right)\\
		&\geq c + \frac{1}{2} \E_{\vec{\sigma}}\left[\sum_{i = 1}^N \sigma_i\right]\\
		&= c.\end{align*}
		
		To this end, given an assignment of the variables $\sigma_1, \dots, \sigma_N \in \{-1,1\}$, let $(b_1, \dots, b_N) \in \{0,1\}^N$ be defined such that \[b_i = \begin{cases} 0 &\text{if } \sigma_i = 1\\
		1 &\text{otherwise} \end{cases}\] and let \[r_0 = 2\pi\gamma^p \left(\sum_{j = 1}^N \alpha^jb_j + \alpha^{N+1}\right).\] Since $0 < r_0 < 2\pi\gamma^p \sum_{j = 1}^{\infty} \alpha^j = \frac{2\pi\gamma^p \alpha}{1 - \alpha} \leq \gamma^p$, $r_0$ is an element of the parameter space $\left(0, \gamma^p\right]$. The inequality $\frac{2\pi\gamma^p \alpha}{1 - \alpha} \leq \gamma^p$ holds because $\alpha \leq \frac{1}{2\pi + 1}$, so $\frac{2\pi\alpha}{1 - \alpha} \leq 1$.
		
		Next, we evaluate $f_{r_0}(x_i) = \frac{1}{2}\left(1 + \cos(r_0x_i)\right)$:
		\begin{align}
		\frac{1}{2}\left(1 + \cos(r_0x_i)\right) &= \frac{1}{2} + \frac{1}{2}\cos\left(2\pi\gamma^p \left(\sum_{j = 1}^N \alpha^jb_j + \alpha^{N+1}\right) \frac{\alpha^{-i}}{2 \gamma^p}\right)\nonumber\\
		&= \frac{1}{2} + \frac{1}{2}\cos\left(\pi\left(\sum_{j = 1}^N \alpha^jb_j + \alpha^{N+1}\right) \alpha^{-i}\right)\nonumber\\
		&= \frac{1}{2} + \frac{1}{2}\cos\left(\sum_{j = 1}^{i-1} \alpha^{j-i}\pi b_j + \pi b_i + \sum_{j = i+1}^N \alpha^{j-i}\pi b_j + \alpha^{N+1-i}\pi\right) \nonumber\\
		&= \frac{1}{2} + \frac{1}{2}\cos\left(\pi\left(b_i + \sum_{j = 1}^{N-i} \alpha^{j}b_{i+j} + \alpha^{N+1-i}\right)\right).\label{eq:cos}
		\end{align} The final equality holds because for every $j <i$, $\alpha^{j-i}$ is a positive power of 2, so $\alpha^{j-i}\pi b_j$ is a multiple of $2\pi$. We will use the following fact: since \[0 < \sum_{j = 1}^{N-i} \alpha^{j}b_{i+j} + \alpha^{N+1-i} \leq \sum_{j = 1}^{N-i+1} \alpha^{j} < \sum_{j = 1}^{\infty} \alpha^{j} = \frac{\alpha}{1 - \alpha},\] the argument of $\cos(\cdot)$ in Equation~\eqref{eq:cos} lies strictly between $\pi b_i$ and $\pi b_i + \frac{\pi\alpha}{1 - \alpha}.$
		
		Suppose $b_i = 0$. Since $\alpha \leq \frac{1}{2}$, we know that $\frac{\pi\alpha}{1 - \alpha} \leq \pi$. Therefore, $\cos(\cdot)$ is monotone decreasing on the interval $\left[0, \frac{\pi\alpha}{1 - \alpha}\right]$. Moreover, we know that $\frac{1}{2}\cos \frac{\pi\alpha}{1 - \alpha} \geq c$. Therefore, $f_{r_0}(x_i) = \frac{1}{2}\left(1 + \cos(r_0x_i)\right) \geq \frac{1}{2} + c$. Since $b_i = 0$, it must be that $\sigma_i = 1$, so $\sigma_i f_{r_0}(x_i) \geq c + \frac{1}{2} = c + \frac{\sigma_i}{2}$. Meanwhile, suppose $b_i = 1$. The function $\cos(\cdot)$ is monotone increasing on the interval $\left[\pi, \pi + \frac{\pi\alpha}{1 - \alpha}\right]$. Moreover, $\frac{1}{2}\cos\left(\pi + \frac{\pi\alpha}{1 - \alpha}\right) = -\frac{1}{2}\cos \frac{\pi\alpha}{1 - \alpha} \leq -c$. Therefore, $f_{r_0}(x_i) = \frac{1}{2}\left(1 + \cos(r_0x_i)\right) \leq \frac{1}{2} - c$. Since $b_i = 1$, it must be that $\sigma_i = -1$, so $\sigma_i f_{r_0}(x_i) \geq c - \frac{1}{2} = c + \frac{\sigma_i}{2}$. Since this is true for any $i \in [N]$, we have that \[\frac{1}{2N} \sum_{i = 1}^N \sigma_i \left(1 + \cos\left(r_0x_i\right)\right) \geq c + \frac{1}{2} \sum_{i = 1}^N \sigma_i,\] as claimed.
	\end{proof}
	We conclude this proof by showing that $\wcrad_N(\ell \circ \cF) \leq \frac{1}{2}$.
	\begin{claim}\label{claim:rad_upper}
		For any $N \geq 1$, $\wcrad_N(\ell \circ \cF) \leq \frac{1}{2}$.
	\end{claim}
	\begin{proof}[Proof of Claim~\ref{claim:rad_upper}] Let $\sample = \left\{\left(x_1, y_1\right), \dots, \left(x_N, y_N\right)\right\} \subset \left[\frac{1}{2\gamma^p}, \infty\right) \times [0,1]$ be an arbitrary set of points. For any assignment of the variables $\sigma_1, \dots, \sigma_N \in \{-1,1\}$, since $\left|f_r\left(x_i\right) - y_i\right| \in [0,1]$, \[\sup_{r \in \left(0, \gamma^p\right]}\sum_{i = 1}^N \sigma_i \left|f_r\left(x_i\right) - y_i\right| \leq \sum_{i = 1}^N \textbf{1}_{\left\{\sigma_i = 1\right\}}.\] Therefore, \[\wcrad_N(\ell \circ \cF) = \sup_{\left(x_1, y_1\right), \dots, \left(x_N, y_N\right)}\frac{1}{N} \E_{\vec{\sigma}}\left[\sup_{r \in \left(0, \gamma^p\right]}\sum_{i = 1}^N \sigma_i \left|f_r\left(x_i\right) - y_i\right|\right] \leq \frac{1}{N} \E_{\vec{\sigma}}\left[\sum_{i = 1}^N \textbf{1}_{\left\{\sigma_i = 1\right\}}\right] = \frac{1}{2},\] as claimed.
	\end{proof}
	Together, Claims~\ref{claim:rad_lower} and \ref{claim:rad_upper} imply that for every $N \geq 1$, $\wcrad_N(\ell \circ \cF) = \frac{1}{2}$.
\end{proof}

%% file: appendix_statistical.tex
\begin{theorem}\label{thm:approximable}
	Let $\cF = \left\{f_{\vec{r}} \mid \vec{r} \in \cR\right\} \subseteq [0,1]^{\cX}$ and $\cG = \left\{g_{\vec{r}} \mid \vec{r} \in \cR\right\} \subseteq [0,1]^{\cX}$ be two sets of functions.
	Suppose the dual class $\cG^*$ $(\gamma, \infty)$-approximates the dual class $\cF^*$. If $\cG$ is statistically learnable, then $\cF$ is $\gamma$-statistically learnable.
\end{theorem}
\begin{proof}
	We will prove that for all integers $N \geq 1$, \[\cV_N(\cF) = \inf_{\cA}\sup_{\dist} \E_{\sample \sim \dist^N}\left[L_{\dist}(\cA_{\sample}) - \inf_{\vec{r} \in \cR}L_{\dist}\left(f_{\vec{r}}\right)\right] \leq \cV_N(\cG) + \gamma.\] Since $\lim_{N \to \infty} \cV_N(\cG) = 0$, this implies that $\lim_{N \to \infty} \cV_N(\cF) \leq \gamma$.
	
	To this end, fix an arbitrary learning algorithm $\bar{\cA}: (\cX \times [0,1])^N \to [0,1]^{\cX}$, distribution $\bar{\dist}$ over $\cX \times [0,1]$, element $\bar{x} \in \cX$, and parameter vector $\bar{\vec{r}} \in \cR$. Since the dual class $\cG^*$ $(\gamma, \infty)$-approximates the dual class $\cF^*$, we know that $\left|g_{\bar{x}}^*\left(\bar{\vec{r}}\right) - f_{\bar{x}}^*\left(\bar{\vec{r}}\right)\right| = \left|g_{\bar{\vec{r}}}\left(\bar{x}\right) - f_{\bar{\vec{r}}}\left(\bar{x}\right) \right| \leq \gamma.$ Since this inequality holds for all $\bar{x} \in \cX$, we also have that \begin{align*}L_{\bar{\dist}}\left(f_{\bar{\vec{r}}}\right) &= \E_{(x,y) \sim \bar{\dist}}\left[\left|f_{\bar{\vec{r}}}(x) - y\right|\right]\\
	&= \E_{(x,y) \sim \bar{\dist}}\left[\left|g_{\bar{\vec{r}}}(x) - y - \left(g_{\bar{\vec{r}}}(x) - f_{\bar{\vec{r}}}(x)\right)\right|\right]\\
	&\geq \E_{(x,y) \sim \bar{\dist}}\left[\left|g_{\bar{\vec{r}}}(x) - y\right| - \left|g_{\bar{\vec{r}}}(x) - f_{\bar{\vec{r}}}(x)\right|\right]\\
	&\geq L_{\bar{\dist}}\left(g_{\bar{\vec{r}}}\right) - \gamma\\
	&\geq \inf_{\vec{r} \in \cR}L_{\bar{\dist}}\left(g_{\vec{r}}\right) - \gamma.\end{align*} These inequalities holds for all parameter vectors $\bar{\vec{r}} \in \cR$, which implies that $\inf_{\vec{r} \in \cR}L_{\bar{\dist}}\left(f_{\vec{r}}\right) \geq \inf_{\vec{r} \in \cR}L_{\bar{\dist}}\left(g_{\vec{r}}\right) - \gamma.$ Therefore, \begin{align*}
	\E_{\sample \sim \bar{\dist}^N}\left[L_{\bar{\dist}}\left(\bar{\cA}_{\sample}\right) - \inf_{\vec{r} \in \cR}L_{\bar{\dist}}\left(f_{\vec{r}}\right)\right]&\leq \E_{\sample \sim \bar{\dist}^N}\left[L_{\bar{\dist}}\left(\bar{\cA}_{\sample}\right) - \inf_{\vec{r} \in \cR}L_{\bar{\dist}}\left(g_{\vec{r}}\right)\right] + \gamma\\
	&\leq \sup_{\dist}\E_{\sample \sim \dist^N}\left[L_{\dist}\left(\bar{\cA}_{\sample}\right) - \inf_{\vec{r} \in \cR}L_{\dist}\left(g_{\vec{r}}\right)\right] + \gamma.
	\end{align*}
	Since this inequality holds for every distribution $\dist$, we have that \[\sup_{\dist}\E_{\sample \sim \dist^N}\left[L_{\dist}\left(\bar{\cA}_{\sample}\right) - \inf_{\vec{r} \in \cR}L_{\dist}\left(f_{\vec{r}}\right)\right] \leq \sup_{\dist}\E_{\sample \sim \dist^N}\left[L_{\dist}\left(\bar{\cA}_{\sample}\right) - \inf_{\vec{r} \in \cR}L_{\dist}\left(g_{\vec{r}}\right)\right] + \gamma.\] Therefore, \[\cV_N(\cF) = \inf_{\cA}\sup_{\dist} \E_{\sample \sim \dist^N}\left[L_{\dist}\left(\cA_{\sample}\right) - \inf_{\vec{r} \in \cR}L_{\dist}\left(f_{\vec{r}}\right)\right]\leq \sup_{\dist}\E_{\sample \sim \dist^N}\left[L_{\dist}\left(\bar{\cA}_{\sample}\right) - \inf_{\vec{r} \in \cR}L_{\dist}\left(g_{\vec{r}}\right)\right] + \gamma.\] Finally, since this inequality holds for every learning algorithm $\bar{\cA}$, we have that 
	\[\cV_N(\cF) \leq \inf_{\cA}\sup_{\dist} \E_{\sample \sim \dist^N}\left[L_{\dist}\left(\cA_{\sample}\right) - \inf_{\vec{r} \in \cR}L_{\dist}\left(g_{\vec{r}}\right)\right] + \gamma = \cV_N(\cG) + \gamma,\] as claimed.
\end{proof}

However, this positive result, Theorem~\ref{thm:approximable}, fails to hold when $L^p$-norm defining the approximation guarantee is not the $L^{\infty}$-norm.

\begin{theorem}\label{thm:not_learnable_statistical}
	For any $\gamma \in (0,1/4)$ and any $p \in [1, \infty)$, there exist function classes $\cF, \cG \subset [0,1]^{\cX}$ with the following properties:
	\begin{enumerate}
		\item The dual class $\cG^*$ $(\gamma, p)$-approximates the dual $\cF^*$.
		\item The class $\cG$ is statistically learnable.
		\item The class $\cF$ is not $\gamma$-statistically learnable.
	\end{enumerate}
\end{theorem}

\begin{proof} The function classes $\cF$ and $\cG$ are the same as those in Theorem~\ref{thm:not_learnable}. Let $t = \gamma^p$, $a = \gamma^{-p} / 2$, $\cR = (0,t]$, and $\cX = [a, \infty)$. For any $r \in \cR$ and $x \in \cX$, let $f_r(x) = \frac{1}{2}(1 + \cos(rx))$ and $\cF = \left\{ f_r \mid r \in \cR\right\}$. For any $r \in \cR$ and $x \in \cX$, let $g_r(x) = \frac{1}{2}$ and $\cG = \left\{g_{r} \mid r \in \cR\right\}$.
	
In Lemma~\ref{lem:approximates}, we prove that the dual class $\cG^*$ $(\gamma, p)$-approximates the dual class $\cF^*$. From Lemma~\ref{lem:not_learnable} in Appendix~\ref{app:learnablility}, we know that for every $N \geq 1$, $\wcrad_N(\ell \circ \cF) = \frac{1}{2}.$
	Therefore, by Theorem~\ref{thm:rad_lb}, $\cV_N(\cF) \geq \frac{1}{4} > \gamma$, so $\cF$ is not $\gamma$-statistically learnable.
\end{proof}
Theorem~\ref{thm:not_learnable_statistical} implies, for example, that even if every function $f_x^* \in \cF^*$ is close to the corresponding function $g_x^* \in \cG^*$ \emph{on average} over the parameter vectors $\vec{r} \in \cR$, the function class $\cF$ still may not be statistically learnable.